\documentclass[journal]{IEEEtran}

\usepackage{amsmath,amsfonts,amssymb,amsthm}
\usepackage{tikz}
\usetikzlibrary{arrows.meta,shapes,positioning,fit,backgrounds,calc}
\usepackage{graphicx}
\usepackage{algorithm}
\usepackage{algpseudocode}
\usepackage{booktabs}
\usepackage{array}
\usepackage{multirow}
\newcolumntype{L}[1]{>{\raggedright\arraybackslash}p{#1}}
\newcolumntype{C}[1]{>{\centering\arraybackslash}p{#1}}
\usepackage{xcolor}
\usepackage[hidelinks]{hyperref}
\usepackage{caption}
\usepackage{cite}
\usepackage{enumitem}

\newtheorem{proposition}{Proposition}
\newtheorem{theorem}{Theorem}
\newtheorem{corollary}{Corollary}[theorem]
\newtheorem{lemma}{Lemma}
\newtheorem{remark}{Remark}
\newtheorem{definition}{Definition}

\usepackage{comment}

\title{\textsc{TierSecAgg}: Tier-Stratified Secure Aggregation
for SignSGD-Based Hierarchical Federated Learning
in IoT Edge--Cloud Systems}

\author{Hyeong-Gun Joo,~\IEEEmembership{Member,~IEEE,} Songnam Hong,~\IEEEmembership{Senior Member,~IEEE,} and Dong-Joon Shin,~\IEEEmembership{Senior Member,~IEEE}%
\thanks{The authors are with the Department of Electronic Engineering, Hanyang University, Seoul, South Korea, Corresponding author: Dong-Joon Shin, e-mail: djshin@hanyang.ac.kr.}%
}

\begin{document}

\maketitle

\begin{abstract}
Hierarchical federated learning (HFL) with sign-based majority vote (signSGD-MV) is well suited to bandwidth-constrained IoT edge--cloud systems, but its secure realization has so far been confined to flat single-server protocols, which do not extend to the hierarchical setting. The majority vote is non-linear, so additive secure aggregation reveals the exact vote tally rather than only the majority bit. Applying a flat protocol within each cluster also exposes each cluster's vote to the cloud, while intermediate edges that are not mutually trusted can fabricate the single weighted value each reports.
Moreover, since the device--edge and edge--cloud boundaries face different trust regimes, no single flat protocol can cover both. We therefore present \textsc{TierSecAgg}, which decomposes secure aggregation for HFL signSGD-MV into two tiers---device--edge (DE) and edge--cloud (EC)---connected through a one-directional, information-minimal interface. 
The DE tier introduces single-mask power sharing to realize per-cluster HFL operation, providing device-level privacy and joint dropout/Byzantine-share recovery over heterogeneous and round-varying active sets. 
The EC tier removes this cluster-vote exposure and counters report fabrication by Shamir-sharing each edge's OTP-masked weighted output itself as the secret, together with an online consistency check and MDS decoding. Among the examined OTP/Shamir candidates, a strawman analysis shows that this is the only commitment-free design satisfying all requirements. The resulting framework jointly provides privacy, dropout recovery, and aggregation-layer Byzantine recovery without adversary-specific verification subprotocols, and a fabrication taxonomy separates the edge behaviors that are detected or corrected from those whose impact is only provably bounded. Experiments on standard and IoT-relevant workloads confirm that the secure realization preserves plaintext learning behavior under honest execution, while the resulting communication and threshold trade-offs are analytically characterized.
\end{abstract}

\begin{IEEEkeywords}
Hierarchical federated learning, secure aggregation,
signSGD majority vote, tier-stratified architecture.
\end{IEEEkeywords}

\section{Introduction}
\label{sec:intro}

IoT deployments such as smart cities, connected healthcare, and industrial monitoring generate privacy-sensitive data at constrained devices, while intermediate edge servers operate under a trust regime distinct from that of the cloud. Federated learning (FL) on such systems must reconcile a severe device-side bandwidth budget with the need to keep aggregated data private as it crosses the edge layer. Hierarchical FL (HFL) matches this edge--cloud structure by placing edges as intermediate aggregators~\cite{kazemi2026hierarchical}, and sign-based communication with edge-level majority vote (signSGD-MV)~\cite{bernstein2018signsgd,safaryan2021stochastic} meets the device-side bandwidth budget by reducing each device's uplink to one bit per coordinate, leaving privacy across the edge layer as the open requirement.

Existing secure-aggregation work suitable for signSGD-MV is confined to Hi-SAFE~\cite{joo2024hisafe}, a flat semi-honest protocol.
Generic additive SecAgg~\cite{bonawitz2017practical,bell2020secure,kadhe2020fastsecagg,so2022lightsecagg} cannot be applied directly because majority vote is non-linear, and a ``reconstruct-then-sign'' detour reconstructs the sum of the sign votes, which reveals the exact vote tally rather than only the one-bit majority output. Beyond these flat-setting obstacles, the hierarchical setting is not a routine extension. Per-cluster outputs become visible to the cloud. Intermediate edges, often run by different organizations, are mutually distrusting and can fabricate the single value they report. In addition, aggregation is weighted by heterogeneous cluster sizes.

A key limitation of existing FL security mechanisms is that privacy, dropout resilience, and aggregation-layer Byzantine robustness are rarely achieved within a single aggregation architecture. Dropout-resilient secure aggregation schemes~\cite{bonawitz2017practical,bell2020secure,kadhe2020fastsecagg,so2022lightsecagg} protect individual updates and tolerate client failures, but do not prevent poisoned or Byzantine shares from corrupting the recovered aggregate. Conversely, Byzantine-robust aggregation rules~\cite{blanchard2017machine,yin2018byzantine,pillutla2022robust,cao2021fltrust} require access to individual updates or their statistics, which is fundamentally at odds with secure aggregation. Recent privacy-preserving Byzantine-robust FL protocols address this conflict by introducing adversary-specific verification or secure-computation mechanisms such as commitments, verifiable secret sharing, zero-knowledge proofs, and coded-computation checks~\cite{so2020byzantine,burkhalter2023rofl,jahani2023byzsecagg,mai2024rflpa,xia2024byitfl,fan2025byzsfl}. However, they are layered on top of the underlying secure aggregation rather than obtained from its own structure, and they target additive aggregation rather than signSGD majority vote. In contrast, to the best of our knowledge, this work is the first to integrate all three requirements into a single tier-stratified HFL architecture for signSGD-MV, without relying on additional adversary-specific verification subprotocols.

We design and analyze secure aggregation for HFL signSGD-MV as a two-tier architecture connected by a one-directional, information-minimal interface. The DE tier is specified abstractly through an ideal functionality that returns only the cluster majority vote and provides constant-round operation, $t$-privacy, and joint dropout/Byzantine tolerance. This abstraction enables modular composition with any EC-tier protocol that consumes only the cluster-level output. Its realization exploits a structural reduction specific to signSGD-MV. 
Since the cluster aggregate is linear in the device inputs, the devices obtain their shares by local addition, and the majority vote is a univariate polynomial of this single shared aggregate. Therefore, secure evaluation reduces to masked-opening polynomial evaluation using preshared mask powers~\cite{cho2018,couteau2019,reisert2024}.
We call this construction single-mask power sharing and realize it at per-cluster HFL granularity over heterogeneous round-varying active sets, cluster-local thresholds, and the DE/EC interface. 

At the EC tier, where the hierarchy-specific security surface arises, each edge Shamir-shares its one-time-pad (OTP) masked weighted output \emph{itself} as the secret. Then, all edges run a consistency check, and the cloud applies maximum-distance-separable (MDS) decoding to recover the masked output without learning the unmasked value.
The novelty of our work lies in the tier-stratified decomposition, the single-mask power-sharing realization of the DE tier, and the identification of an EC-tier security surface that prior work leaves open: hiding each edge's output from the very party that reconstructs it while still recovering a consistent value from every edge, as clarified by the strawman analysis in Section~\ref{sec:strawman}.

\subsection{Contributions}
\begin{itemize}[leftmargin=*,itemsep=2pt]
  \item \textbf{Tier-stratified architecture.} We decompose secure aggregation for HFL signSGD-MV into a DE tier and an EC tier connected by a one-directional, information-minimal interface, enabling independent analysis and modular composition.

  \item \textbf{Hierarchical DE-tier realization.} We introduce single-mask power sharing, a per-cluster adaptation of masked-opening polynomial evaluation with preshared mask powers~\cite{cho2018,couteau2019,reisert2024}, and realize the DE-tier ideal functionality $\mathcal{F}_{\mathrm{DE}}$ with it. This construction supports parallel per-cluster instances over heterogeneous, round-varying active sets with cluster-local privacy and recovery thresholds, while exposing only the cluster majority output to the EC tier.

  \item \textbf{Commitment-free EC-tier protocol.} We identify an EC-tier security surface that prior work does not jointly address: hiding each edge's weighted output from the cloud, which acts as the reconstructor, while recovering a consistent value from each edge against malicious intermediate edges under the non-linear signSGD majority vote. We construct the only commitment-free OTP/Shamir/MDS protocol, among the examined candidates, that meets both requirements, providing information-theoretic per-edge privacy under a cloud--edge non-collusion assumption and share-consistency recovery.
\end{itemize}

\subsection{Organization of the Paper}
Section~\ref{sec:related} introduces previous works related to our results. Section~\ref{sec:system} presents the system and threat models, and Section~\ref{sec:de_tier} specifies the DE tier through an ideal functionality and realizes it by single-mask power sharing. The EC-tier protocol is proposed in Section~\ref{sec:ec_protocol} together with its strawman justification and the fabrication taxonomy. Section~\ref{sec:composition} establishes compositional security. Section~\ref{sec:experiments} evaluates the proposed method on standard and IoT workloads. Finally, Section~\ref{sec:discussion} concludes the paper.

\section{Related Work}
\label{sec:related}

\paragraph*{The narrow signSGD-MV secure aggregation line}
Only one prior protocol targets signSGD-MV directly. Hi-SAFE~\cite{joo2024hisafe} uses Beaver triples for secure majority-vote evaluation under a flat semi-honest model without dropout tolerance or Byzantine robustness. The present paper develops single-mask power sharing with MDS-based recovery for solving this problem, together with the EC-tier protocol, the fabrication taxonomy, the tier-stratified composition framework, and the complete per-cluster DE-tier realization.

\paragraph*{Hierarchical additive secure aggregation}
A recent information-theoretic secure aggregation line studies hierarchical secure aggregation (HSA) over a user--relay--server network, enforcing relay security such that relays remain oblivious to user inputs and characterizing the minimum per-hop communication and the minimum correlated key randomness required for such security~\cite{zhang2024optimalhsa,zhang2025cyclic,xu2025hsacollusion,lu2026groupwise}. This line is additive-only and semi-honest, providing no share-consistency mechanism against malicious intermediaries. Also, a ``reconstruct-then-sign'' detour over additive SecAgg~\cite{bonawitz2017practical,bell2020secure} further reveals the exact vote tally, and is hence not a substitute for signSGD-MV.

\paragraph*{Other privacy-preserving FL}
DP-based FL~\cite{geyer2017differentially,mcmahan2018learning} bounds inference risk through a formal privacy budget, but the injected noise can degrade the sign vote (Remark~\ref{rem:inherent_leak}). HE-based FL~\cite{aono2017privacy,zhang2020batchcrypt} hides individual updates under encryption, but ciphertext expansion inflates device uploads by orders of magnitude, which is impractical for constrained IoT uplinks. Verifiable/robust FL~\cite{burkhalter2023rofl,guo2020verifl}, including verifiable masking-based HFL~\cite{vflgc2025chaining,jnca2025verifiable}, focuses on additive aggregation using commitments or zero-knowledge proofs, leaving the non-linear vote and EC-tier privacy unaddressed. TEE-based HFL~\cite{mo2021ppfl} requires a trusted hardware enclave at every edge.

\paragraph*{Low-latency secure polynomial evaluation}
Constant-round evaluation of a polynomial on a secret-shared input is a classical MPC problem with two relevant masking families. Multiplicative masking~\cite{barilan1989,damgard2006constant,lu2022polymath} requires only linear precomputation, but a multiplicatively masked zero opens to zero regardless of the mask, so the input must be guaranteed nonzero. The MV aggregate violates this requirement, since a sign vote among an even number of participants can tie, making the aggregate exactly zero. Additive masking with preshared mask powers, termed binomial tuples in~\cite{reisert2024} and used in passively secure form in~\cite{cho2018,couteau2019}, opens one masked value per variable and evaluates locally. The online mechanism of our single-mask power sharing coincides with the univariate case of this family, and we do not claim novelty for the online mechanism itself. However, this line operates over additive $n$-out-of-$n$ sharings, whose openings cannot tolerate even a single missing share. Our DE-tier contribution is instead to instantiate this approach with honest-majority Shamir sharing and a DN-based offline phase~\cite{damgard2007}, to integrate MDS decoding into both openings, to show that the MV polynomial is univariate in the aggregate and therefore additive masking remains correct at $x=0$, and to establish the per-cluster HFL reduction together with its dropout/Byzantine recovery analysis. We use single-mask power sharing as the name for this combined construction.

\paragraph*{MDS decoding and dealer binding in MPC}
Robust reconstruction via error correction is standard in honest-majority MPC~\cite{damgard2007,atlas2021}, where it ensures computation \emph{correctness}, while dealer \emph{binding} traditionally relies on commitment-based VSS. Our EC tier uses neither in its classical role. In our EC tier, each edge acts as the Shamir dealer of its own masked output, and we use dealer exclusively in this sense. We combine a lightweight random linear combination check for dealer well-formedness with MDS decoding repurposed for share-level recovery over OTP-masked values. This combination operates in an asymmetric setting, where multiple edge dealers face a single reconstructing cloud and the recovered values feed a weighted global aggregate. To our knowledge, it has not been considered in prior work.

Table~\ref{tab:related_merged} compares the FL lines discussed above in terms of security requirements. Dropout-resilient SecAgg provides input privacy and dropout tolerance but not Byzantine robustness. Byzantine-robust aggregation rules provide robustness but require access to individual updates or update statistics, conflicting with SecAgg privacy. Privacy-preserving robust FL bridges this privacy--robustness conflict through adversary-specific verification subprotocols layered on top of the aggregation pipeline and not tailored to hierarchical signSGD majority vote. To our knowledge, no existing method simultaneously supports constant-round secure majority-vote evaluation, joint DE-tier dropout/Byzantine recovery, and EC-tier per-edge privacy with share-consistency recovery. \textsc{TierSecAgg} fills this gap through a single tier-stratified architecture. 

\begin{table}[t]
\centering
\caption{Comparison of FL protocols by security requirements ($\checkmark$/$\triangle$/$\times$: yes/partial/no). Priv./Drop./Byz.: input privacy, dropout tolerance, Byzantine robustness; MV: constant-round secure majority vote; EC: EC-tier per-edge privacy and share-consistency recovery.}
\label{tab:related_merged}
\renewcommand{\arraystretch}{1.1}
\setlength{\tabcolsep}{3pt}
\footnotesize
\begin{tabular}{@{}L{3.9cm}C{0.75cm}C{0.75cm}C{0.75cm}C{0.7cm}C{0.7cm}@{}}
\toprule
\textbf{Method class} &
\textbf{Priv.} & \textbf{Drop.} & \textbf{Byz.} &
\textbf{MV} & \textbf{EC} \\
\midrule
Dropout-resilient additive SecAgg~\cite{bonawitz2017practical,bell2020secure,kadhe2020fastsecagg,so2022lightsecagg}
  & $\checkmark$ & $\checkmark$ & $\times$ & $\times$ & $\times$ \\
Byzantine-robust FL aggregation~\cite{blanchard2017machine,yin2018byzantine,pillutla2022robust,cao2021toward}
  & $\times$ & $\triangle$ & $\checkmark$ & $\times$ & $\times$ \\
Privacy-preserving robust FL~\cite{so2020byzantine,burkhalter2023rofl,jahani2023byzsecagg,mai2024rflpa,xia2024byitfl,fan2025byzsfl,guo2020verifl}
  & $\checkmark$ & $\triangle$ & $\checkmark^{\ddagger}$ & $\times$ & $\times$ \\
DP/HE-based FL~\cite{geyer2017differentially,mcmahan2018learning,aono2017privacy,zhang2020batchcrypt}
  & $\checkmark$ & $\triangle$ & $\times$ & $\triangle$ & $\triangle$ \\
Hierarchical additive SecAgg~\cite{xu2025hsacollusion, zhang2024optimalhsa,zhang2025cyclic, lu2026groupwise}
  & $\checkmark$ & $\triangle$ & $\times$ & $\times$ & $\triangle$ \\
Verifiable masking-based HFL~\cite{vflgc2025chaining,jnca2025verifiable}
  & $\checkmark$ & $\triangle$ & $\triangle^{\ddagger}$ & $\times$ & $\triangle$ \\
TEE-based HFL~\cite{mo2021ppfl}
  & $\checkmark$ & $\triangle$ & $\triangle$ & $\checkmark$ & $\triangle$ \\
Hi-SAFE~\cite{joo2024hisafe}
  & $\checkmark$ & $\times$ & $\times$ & $\checkmark$ & $\times$ \\
\textbf{\textsc{TierSecAgg}}
  & $\checkmark$ & $\checkmark$ & $\checkmark^{\dagger}$ & $\checkmark$ & $\checkmark$ \\
\bottomrule
\multicolumn{6}{@{}l}{\footnotesize
$^{\dagger}$Aggregation-layer recovery only.} \\
\multicolumn{6}{@{}l}{\footnotesize
$^{\ddagger}$Via adversary-specific verification subprotocols (ZKP,
VSS, commitments).}
\end{tabular}
\end{table}

\section{System Model and Threat Model}
\label{sec:system}

\subsection{Network and Aggregation}

The system consists of three components, namely devices, $Q$ edge
servers ($\mathcal{V}=\{1,\ldots,Q\}$), and a single cloud, organized
into two tiers, the DE tier and the EC tier
(Fig.~\ref{fig:system}).
Edge $q$ manages $M_q$ registered devices, of which $n_q^\tau\le M_q$ are
\emph{active} in round $\tau$ (active set $S_q^\tau$), abbreviated $n_q$ when the active set is stable across rounds, with
non-negative weight $D_q$.
Each active device $i$ holds a sign vector
$x_i(\tau)=\mathrm{sign}(g_i(\tau))\in\{-1,+1\}^{d}$ as its private
input, where $d$ is the model dimension and all protocol operations act
coordinate-wise.
The cluster majority is
\begin{equation} 
  \tilde{g}_q(\tau)=F_q\bigl(x_q(\tau)\bigr),
  \quad
  x_q(\tau)=\sum_{i\in S_q^\tau}x_i(\tau),
  \label{eq:cluster_majority}
\end{equation}
where $F_q(\cdot)=\mathrm{sign}(\cdot)$ is realized as a polynomial
over a prime field (Section~\ref{sec:de_concrete}).
The cloud computes
$Z(\tau)=\sum_q D_q\hat{g}_q(\tau)$ and
$\tilde{g}^{\mathrm{glob}}(\tau)=\mathrm{sign}(Z(\tau))$,
where $\hat{g}_q(\tau)$ is the value recovered at the EC tier (equal to
$\tilde{g}_q(\tau)$ for honest edges).

\begin{figure}[t]
\centering
\includegraphics[width=1.0\columnwidth]{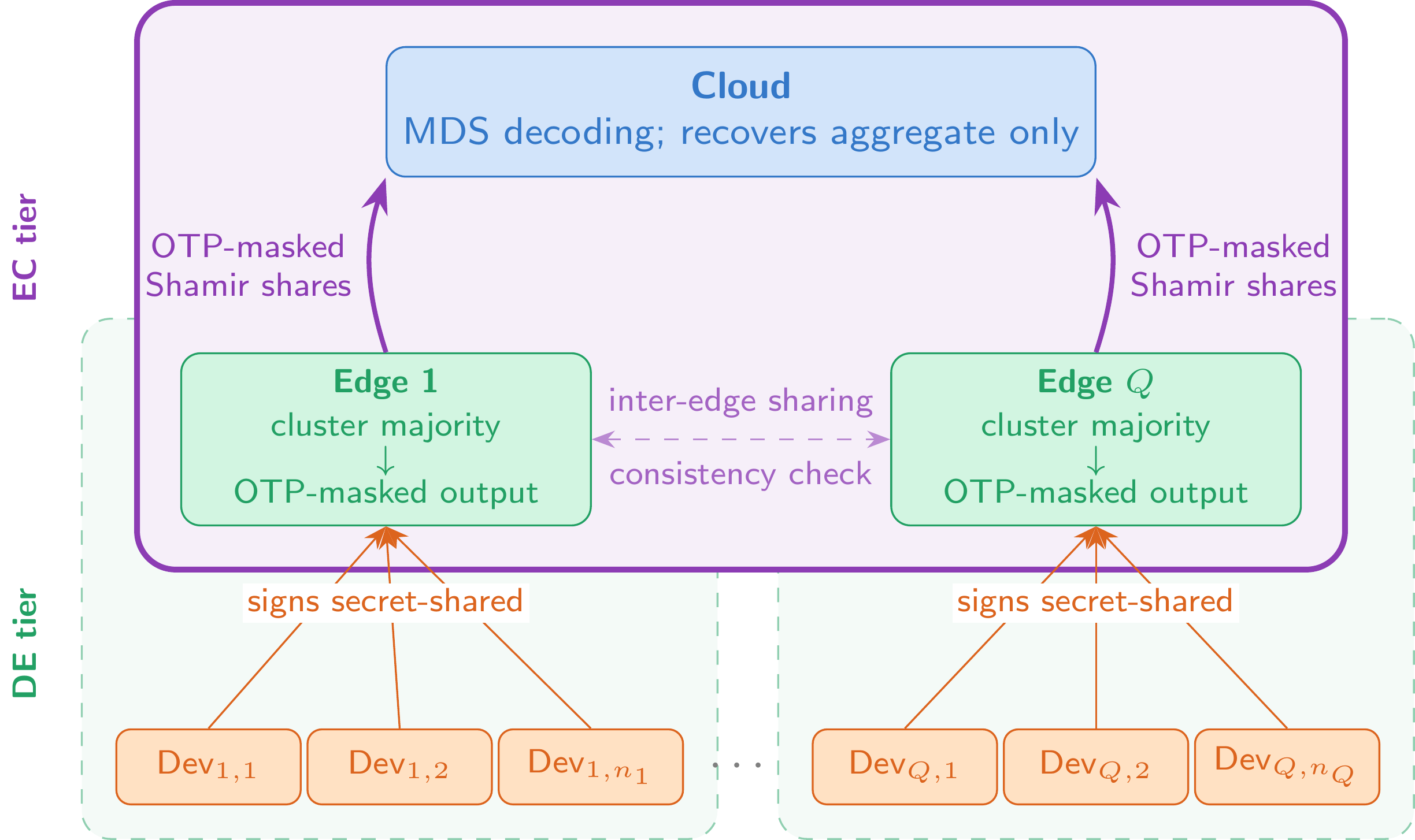}
\caption{\textsc{TierSecAgg} architecture: the DE tier evaluates the
cluster majority via secure majority-vote polynomial evaluation at each
edge, and the EC tier shares the OTP-masked $y_q$ to the cloud with an
inter-edge consistency check, recovered by MDS decoding.}
\label{fig:system}
\end{figure}

\subsection{Threat Model} 
\label{sec:threat}

We separate the adversarial model by role because the same physical edge plays two logically different roles.
At the DE tier the edge is the local reconstruction point of a
cluster-level secure majority protocol.
It is allowed to learn the prescribed output $\tilde{g}_q(\tau)$ but
should learn nothing further about individual device signs.
Since DE-tier edges are operated under service agreements, the realistic threat is passive observation, so the DE edge is modeled as \emph{semi-honest}, while Byzantine tolerance at the DE tier applies to \emph{device} uploads (up to $e$ corrupted device shares and $s$ dropouts).
Following the classical usage~\cite{lamport1982byzantine}, a Byzantine share is any share inconsistent with an honest execution regardless of intent.
In IoT deployments such shares arise not only from compromised devices but also from firmware bugs, memory faults, or desynchronized round state, and dropouts arise routinely from battery depletion, duty cycling, and intermittent connectivity. The MDS-based recovery of Section~\ref{sec:de_tier} is intent-agnostic, treating a faulty and a manipulated share identically.

At the EC tier the edge becomes an intermediate reporting party toward
the cloud.
Edges may be run by different organizations without a basis for mutual trust, so EC-tier edges are modeled as \emph{malicious} with a static corruption set of size at most $t_{\mathrm{EC}}$, capturing equivocation, inconsistent share forwarding, late repudiation, and semantic fabrication.
The malicious model is a worst-case envelope rather than an accusation of intent. It equally covers unintentional deviations with the same signatures, such as an inconsistent dealing left by a crash during Shamir distribution, forwarding corruption from a faulty aggregation stack, or a fabricated-looking report produced by a software regression, so every tolerance claim below holds whether a deviation is adversarial or accidental.
The cloud is semi-honest.
Cross-role deviation by the same physical edge is treated through the
fabrication taxonomy (Section~\ref{sec:fabrication}).
This role separation follows established practice rather than assuming
that the same operator changes nature across tiers.
Mixed-adversary MPC lets passive and active corruptions coexist under
separate thresholds~\cite{fitzi1998trading}, secure aggregation
routinely models the server as semi-honest and the clients as
malicious~\cite{bonawitz2017practical}, and the only element specific
to our setting is that the two models attach to the two protocol roles
of the same physical edge.
This attachment is sound because any active DE-tier deviation that
alters the reported cluster output is equivalent, at the DE/EC
interface, to semantic fabrication of $\hat{g}_q(\tau)$, which the
taxonomy bounds rather than detects.
The semi-honest DE-tier modeling therefore matters only for the
privacy of device inputs, and active attacks aimed at extracting them,
such as broadcasting a manipulated opening, fall outside the model, an
explicit limitation shared with the semi-honest server of flat SecAgg.

\paragraph*{Non-collusion between the cloud and edges}
The information-theoretic per-edge privacy guarantee assumes the cloud
learns none of the fresh pairwise OTP values that the edge pairs
establish per round (Section~\ref{sec:ec_construction}).
Collusion with $c$ corrupted edges exposes the involved OTPs and
degrades the guarantee to the residual honest subset.
The information-theoretic form further requires the per-round values
to be true fresh randomness established outside the cloud's
observation, for example from pre-distributed random tapes.
If they are instead expanded from shared seeds by a pseudorandom
generator, or exchanged over computationally secure channels whose
ciphertexts the cloud can observe, the same argument yields a
computational guarantee, mirroring the computational relaxation of
the DE tier's optional PRSS instantiation
(Remark~\ref{rem:prss}).
The non-collusion assumption itself is standard for pairwise
OTP-based SecAgg~\cite{bonawitz2017practical} and is most directly
justified in the cross-silo setting we target, where edge operators
and the cloud are administratively separated.

\paragraph*{Channels and setup}
DE-tier channels follow $\mathcal{F}_{\mathrm{DE}}$ below.
In the default explicit Shamir realization, authenticated
device-to-device channels are required only for the one-shot
input-sharing layer.
At the EC tier, edges communicate over stable authenticated channels
(e.g., TLS) with the cloud as the reconstruction point, and a private
authenticated pairwise channel per edge pair (or pre-distributed random
tapes) establishes fresh per-round OTP values outside the cloud's
observation.

\section{DE Tier: Ideal Functionality and Single-Mask Power-Sharing Realization}
\label{sec:de_tier}

\subsection{DE-Tier Ideal Functionality $\mathcal{F}_{\mathrm{DE}}$}
\label{sec:fde}

To decouple the EC-tier analysis from any specific DE-tier protocol, we
abstract the DE tier as an ideal functionality.

\begin{definition}[DE-tier ideal functionality $\mathcal{F}_{\mathrm{DE}}$]
\label{def:fde}
For each cluster $q$ and round $\tau$, $\mathcal{F}_{\mathrm{DE}}$ takes the device sign inputs $\{x_i(\tau)\}_{i\in S_q^\tau}$, returns the cluster majority output $\tilde{g}_q(\tau)=F_q(x_q(\tau))\in\{-1,+1\}$ to the edge, and satisfies the following properties:
\begin{itemize}[leftmargin=*,itemsep=1pt]
  \item \textbf{(P1) Constant-round online operation:} A constant number of lightweight field-element uploads per device per round, with no multi-round client--server interaction online.
  \item \textbf{(P2) $t_{\mathrm{DE}}$-privacy:} The view of the semi-honest edge, together with any semi-honest coalition of at most $t_{\mathrm{DE}}<n_q^\tau/2$ devices, is indistinguishable from a simulation given $\tilde{g}_q(\tau)$ and the coalition's own inputs and randomness. Hence, the indistinguishability is statistical for an information-theoretic realization and computational for a pseudorandom one (Remark~\ref{rem:prss}).
  \item \textbf{(P3) Joint dropout/Byzantine tolerance:} The edge recovers the correct $\tilde{g}_q(\tau)$ against any combination of $e$ Byzantine device shares and $s$ dropouts satisfying $2e+s\le d_{\min}^{\mathrm{DE}}-1$, where $d_{\min}^{\mathrm{DE}}$ is the minimum distance of the underlying MDS code. For a degree-$\delta$ opening among $n_q^\tau$ parties, $d_{\min}^{\mathrm{DE}}=n_q^\tau-\delta$.
\end{itemize}
\end{definition}

$\mathcal{F}_{\mathrm{DE}}$ is realized by the construction of
Section~\ref{sec:de_concrete}, and Hi-SAFE~\cite{joo2024hisafe} satisfies P1--P2 but not P3. The end-to-end argument depends only on $\mathcal{F}_{\mathrm{DE}}$ being realized by \emph{some} protocol satisfying Definition~\ref{def:fde}, so the two tiers can be improved and analyzed independently (Section~\ref{sec:composition}). Following the standard hybrid-model approach, we specify the DE tier abstractly and the EC tier concretely.

\subsection{A Concrete DE-Tier Instantiation via Single-Mask Power Sharing}
\label{sec:de_concrete}

Since the cluster aggregate is linear in the device inputs and the majority vote is a univariate polynomial evaluation of this aggregate value, one masked opening suffices for the entire secure evaluation. We realize $\mathcal{F}_{\mathrm{DE}}$ using single-mask power sharing built on this reduction, operating independently for each cluster over heterogeneous round-varying active sets and cluster-local thresholds. The key structural property is that the edge reconstructs \emph{only} $\tilde{g}_q(\tau)$: the device inputs stay secret-shared throughout, and the MV polynomial is evaluated using a single masked opening of the aggregate online.

\paragraph*{Per-cluster formulation}
Each cluster $q$ runs an independent instance over
$\mathbb{F}_{p_{\mathrm{DE}}}$ with $p_{\mathrm{DE}}>n_q^\tau$ for all
rounds.
Since every protocol operation acts coordinate-wise
(Section~\ref{sec:system}), we fix a single model coordinate
throughout this section and write $x_i^{(\tau)}\in\{-1,+1\}$ for
device $i$'s sign on that coordinate, so the per-coordinate form of
\eqref{eq:cluster_majority} becomes
\begin{equation}
  \tilde{g}_q^{(\tau)} = F_q(x_q^{(\tau)}),
  \qquad
  x_q^{(\tau)} = \sum_{i\in S_q^\tau} x_i^{(\tau)},
  \label{eq:de_target}
\end{equation}
where $F_q\colon\mathbb{F}_{p_{\mathrm{DE}}}\to\{-1,+1\}$ realizes
the majority rule exactly via Fermat's Little
Theorem~\cite{joo2024hisafe}: writing
$\mathcal{M}=\{-n_q^\tau,-n_q^\tau+2,\ldots,n_q^\tau\}$ for the
attainable sums,
\begin{equation}
F_q(x)=\!\!\!\sum_{m\in\mathcal{M}}\mathrm{sign}(m)
\bigl[1-(x-m)^{p_{\mathrm{DE}}-1}\bigr]\!\!\pmod{p_{\mathrm{DE}}},
\label{eq:mv_poly}
\end{equation}
where $\mathrm{sign}(0)$ is fixed by the tie-breaking policy for even
$n_q^\tau$.
Each term $1-(x-m)^{p_{\mathrm{DE}}-1}$ is the Fermat indicator of
$x=m$, so $F_q(\cdot)$ is an exact lookup representation of the
majority rule, and $p_{\mathrm{DE}}>n_q^\tau$ embeds all attainable
sums without modular ambiguity.
Expanding \eqref{eq:mv_poly} gives a polynomial of degree
$N\le p_{\mathrm{DE}}-1$, and hence the evaluation requires only the
powers of the single aggregate $x_q^{(\tau)}$ up to degree $N$.

\paragraph*{Offline preparation}
Per round, the cluster jointly precomputes a fresh
degree-$t_{\mathrm{DE}}$ mask sharing $[a]_{t_{\mathrm{DE}}}$, the
power sharings $\{[a^j]_{t_{\mathrm{DE}}}\}_{j=0}^{N}$ via sequential
DN multiplication ($N-1$ input-independent secure
multiplications), and the coefficient sharings
\begin{equation}
  [P_r]_{t_{\mathrm{DE}}}
  =
  \sum_{j=r}^{N} c_j\binom{j}{r}[a^{j-r}]_{t_{\mathrm{DE}}},
  \quad r=0,\ldots,N,
  \label{eq:de_offline}
\end{equation}
where $\{c_j\}$ are the coefficients of $F_q(\cdot)$ in \eqref{eq:mv_poly}.

\paragraph*{Input-sharing layer}
Each active device Shamir-shares its sign input at degree $t_{\mathrm{DE}}$, sending one evaluation to each peer, so that every device holds a share of $[x_i^{(\tau)}]_{t_{\mathrm{DE}}}$ for all $i\in S_q^\tau$ ($n_q^\tau-1$ field elements per device per coordinate). We treat this layer separately from the MV-evaluation layer because it is realization-dependent.
The explicit Shamir realization is the \emph{default}, and the statistical form of the privacy guarantee P2 refers to this realization. An \emph{optional} PRSS-based instantiation~\cite{cramer2005prss} instead derives the input and mask sharings locally from predistributed seeds, which removes the sharing traffic from the per-coordinate online path. The cost is that DE-tier privacy weakens from statistical to computational (Remark~\ref{rem:prss}). In both realizations, device-to-device communication is confined to this layer.
\begin{algorithm}[t]  
\caption{DE tier using single-mask power-sharing (cluster $q$, round
$\tau$, per coordinate)}
\label{alg:de_smps}
\begin{algorithmic}[1]
\Require $\mathbb{F}_{p_{\mathrm{DE}}}$ with
$p_{\mathrm{DE}}>n_q^\tau$, threshold
$t_{\mathrm{DE}}<n_q^\tau/2$, MV coefficients $\{c_j\}_{j=0}^{N}$ for $\sum_{j=0}^N c_j x^j$
\Statex \textit{Offline (input-independent)}
\State sample $[a]_{t_{\mathrm{DE}}}$ and compute
$\{[a^j]_{t_{\mathrm{DE}}}\}_{j=0}^{N}$ by $N-1$ sequential DN
multiplications
\State form $[P_r]_{t_{\mathrm{DE}}}$, $r=0,\ldots,N$, via
\eqref{eq:de_offline}
\Statex \textit{Online}
\State each active device $i$ Shamir-shares $x_i^{(\tau)}$ at degree
$t_{\mathrm{DE}}$
\State each device locally computes
$[x_q^{(\tau)}]\gets\sum_{i\in S_q^\tau}[x_i^{(\tau)}]$ and sends
its share of $[\hat{x}^{(\tau)}]=[x_q^{(\tau)}]-[a]$ to the edge
\State the edge reconstructs $\hat{x}^{(\tau)}$ by MDS decoding and broadcasts it
\State each device submits its share of
$[F_q(x_q^{(\tau)})]$ via \eqref{eq:de_eval}
\State the edge recovers $\tilde{g}_q(\tau)\in\{-1,+1\}$ by MDS decoding
\State \Return $\tilde{g}_q(\tau)$
\end{algorithmic}
\end{algorithm}

\paragraph*{Online phase (single masked opening)}
Given the input sharings above, each round proceeds in three steps.
 
\emph{Step 1 (local aggregation).}
Each device locally computes
$[x_q^{(\tau)}]_{t_{\mathrm{DE}}}
\gets\sum_{i\in S_q^\tau}[x_i^{(\tau)}]_{t_{\mathrm{DE}}}$,
a linear combination of local shares that requires no communication.
 
\emph{Step 2 (mask and open).}
Each device sends its share of
$[\hat{x}^{(\tau)}]_{t_{\mathrm{DE}}}
=[x_q^{(\tau)}]_{t_{\mathrm{DE}}}-[a]_{t_{\mathrm{DE}}}$
to the edge, which robustly reconstructs
$\hat{x}^{(\tau)}=x_q^{(\tau)}-a$ and broadcasts it.
This is the only opening in the online phase, and the revealed value
is statistically independent of $x_q^{(\tau)}$ because the mask $a$
is uniformly random and input-independent.
 
\emph{Step 3 (MV-polynomial evaluation).}
Each device locally forms the public powers
$(\hat{x}^{(\tau)})^0,\ldots,(\hat{x}^{(\tau)})^{N}$ and its share
of the output sharing
\begin{equation}
[F_q(x_q^{(\tau)})]_{t_{\mathrm{DE}}}
=\sum_{r=0}^{N}(\hat{x}^{(\tau)})^r\,[P_r]_{t_{\mathrm{DE}}},
\label{eq:de_eval}
\end{equation}
then submits that share to the edge.
No further interaction is needed, since the powers are public and the
$[P_r]$ are precomputed, and the edge recovers
$\tilde{g}_q(\tau)\in\{-1,+1\}$ by MDS decoding.
 
Every opening above is protected by the same coding structure: a
degree-$\delta$ Shamir sharing among $n$ parties is represented by a codeword of an
$[n,\delta+1,n-\delta]$ MDS code with minimum distance
$d_{\min}=n-\delta$, so the edge uniquely recovers the shared value
against any combination of $e$ Byzantine shares (errors) and $s$
dropouts (erasures) with $2e+s\le d_{\min}-1$, exactly the tolerance
form of P3.
The edge thus sees only $\hat{x}^{(\tau)}$ and $\tilde{g}_q(\tau)$,
never individual inputs or the unmasked aggregate, and the whole
evaluation uses a single masked opening within a constant-round
online phase, matching P1.
Algorithm~\ref{alg:de_smps} summarizes the complete procedure.

\paragraph*{Correctness}
Lemma~\ref{lem:binomial} and Theorem~\ref{thm:mv_correct} establish that Algorithm~\ref{alg:de_smps} generates the intended sharing.
 
\begin{lemma}[Binomial expansion]
\label{lem:binomial}
For any $x\in\mathbb{F}_{p_{\mathrm{DE}}}$ and mask
$a\in\mathbb{F}_{p_{\mathrm{DE}}}$ with $\hat{x}=x-a$,
\[
\sum_{r=0}^{k}\binom{k}{r}\hat{x}^r\,[a^{k-r}]_{t_{\mathrm{DE}}}
=[x^k]_{t_{\mathrm{DE}}}.
\]
\end{lemma}
\begin{proof}
By the binomial theorem, $(\hat{x}+a)^k=\sum_{r=0}^{k}\binom{k}{r}\hat{x}^r a^{k-r}$. Since $\hat{x}+a=x$, the claim follows from the linearity of Shamir sharing.
\end{proof}
 
\begin{theorem}[MV-polynomial evaluation correctness]
\label{thm:mv_correct}
Algorithm~\ref{alg:de_smps} correctly generates $[F_q(x_q^{(\tau)})]_{t_{\mathrm{DE}}}$.
\end{theorem}

\begin{proof}
By Lemma~\ref{lem:binomial} and Eq.~\eqref{eq:de_offline}, writing
$x=x_q^{(\tau)}$ and $\hat{x}=\hat{x}^{(\tau)}$ for brevity,
\begin{align*}
[F_q(x)]_{t_{\mathrm{DE}}}
&=\sum_{j=0}^{N}c_j\,[x^j]_{t_{\mathrm{DE}}}
 =\sum_{j=0}^{N}c_j\sum_{r=0}^{j}\binom{j}{r}
   \hat{x}^r[a^{j-r}]_{t_{\mathrm{DE}}}\\
&=\sum_{r=0}^{N}\hat{x}^r
   \sum_{j=r}^{N}c_j\binom{j}{r}[a^{j-r}]_{t_{\mathrm{DE}}}
 =\sum_{r=0}^{N}\hat{x}^r[P_r]_{t_{\mathrm{DE}}}.
\end{align*}
\end{proof}
The hierarchical protocol runs $Q$ such instances in parallel over
heterogeneous active sets with cluster-local thresholds, and the flat single-server case is recovered as the special case $Q=1$ on a fixed participant set.


\paragraph*{Cost}
The online MV-evaluation layer transmits exactly two field elements per device per coordinate, with a per-device cost independent of $n_q$, giving $O(Qn_q d\log p_{\mathrm{DE}})$ aggregate communication. The input-sharing layer adds $O(Qn_q^2 d\log p_{\mathrm{DE}})$ in the default instantiation, which the optional PRSS instantiation reduces to $O(Qn_q d\log p_{\mathrm{DE}})$.
Offline storage is $O(N)=O(p_{\mathrm{DE}})$ sharings per cluster, versus the exponential $O(2^{N})$ storage of a naive subset-mask alternative that pre-distributes masks for every possible active subset. The offline phase itself communicates: producing $\{[a^j]\}$ requires $N-1$ sequential DN multiplications per coordinate with a fresh mask per round. It gives $O(Qn_qNd\log p_{\mathrm{DE}})$ per-round offline traffic off the online critical path, as quantified in Appendix~\ref{app:additional_exp}.
On the online path, Beaver-triple evaluation of a degree-$N$ polynomial as in Hi-SAFE~\cite{joo2024hisafe} requires $O(N)$ online secure multiplications and openings per coordinate, whereas the construction opens only a single masked value.
The DE-tier threshold $t_{\mathrm{DE}}<n_q^\tau/2$ is governed by the cluster size and is independent of the EC-tier threshold $t_{\mathrm{EC}}<Q/4$ in Section~\ref{sec:ec_construction}. Choosing $t_{\mathrm{DE}}$ strictly below
$n_q^\tau/2$ enlarges the recovery region at the cost of reduced collusion privacy.

\begin{proposition}[DE-tier security]
\label{prop:de_security}
The per-cluster construction above realizes
$\mathcal{F}_{\mathrm{DE}}$ of Definition~\ref{def:fde} with:
\begin{enumerate}[label=(\roman*),leftmargin=*,itemsep=1pt]
  \item \textbf{(P1)} a single masked opening in a constant-round
        online phase, with no multi-round dispute resolution.
  \item \textbf{(P2)} $t_{\mathrm{DE}}$-privacy for the semi-honest
        edge together with any semi-honest coalition of at most
        $t_{\mathrm{DE}}<n_q^\tau/2$ devices, simulatable from
        $\tilde{g}_q(\tau)$ and the coalition's own inputs and
        randomness.
        The simulation is perfect under the default explicit Shamir
        input--sharing realization and computational (up to the PRG advantage)
        under the PRSS realization.
  \item \textbf{(P3)} recovery of $\tilde{g}_q(\tau)$ against $e$
        Byzantine device shares and $s$ dropouts under two operating
        modes: \emph{setup-completed mode}, in which offline preparation
        finishes on a stable set before the online round, with
        $2e+s\le n_q^\tau-t_{\mathrm{DE}}-1$, and
        \emph{conservative end-to-end mode}, in which degree-$2t_{\mathrm{DE}}$
        DN openings on the online critical path, with
        $2e+s\le n_q^\tau-2t_{\mathrm{DE}}-1$.
        Unless otherwise noted, we adopt the conservative bound
        throughout.
\end{enumerate}
\end{proposition}
\begin{proof}
See Appendix~\ref{app:de_proof}.
\end{proof}

\begin{remark}[Optional PRSS instantiation]
\label{rem:prss}
A PRSS instantiation~\cite{cramer2005prss} pre-distributes, once per
stable session, a seed to every maximal unqualified subset
($O(\binom{n_q^\tau}{t_{\mathrm{DE}}})$ seeds per cluster), after which
degree-$t_{\mathrm{DE}}$ random sharings are derived without
communication and each input (and the mask $a$) is provided by a single
broadcast.
This removes the $O(Qn_q^2 d\log p_{\mathrm{DE}})$ term but makes every
shared value a PRG output, so P2 holds only computationally.
It is well suited to the small- or constant-$t_{\mathrm{DE}}$ regime evaluated
here, and P3 and the EC-tier analysis remain unchanged.
\end{remark}

\section{EC Tier: Commitment-Free Protocol and Fabrication Taxonomy}
\label{sec:ec_protocol}

\subsection{EC-Tier Goals}

The EC-tier protocol must achieve, simultaneously and without
cryptographic commitments:
\begin{itemize}[leftmargin=*,itemsep=1pt]
  \item \textbf{(G1) Per-edge output privacy.} The cloud's view reveals
        nothing about individual $\hat{g}_q(\tau)$ beyond what is
        implied by the released aggregate $Z(\tau)$.
  \item \textbf{(G2) Share-consistency recovery.} For up to
        $t_{\mathrm{EC}}$ malicious edges, the cloud either recovers a
        uniquely defined value per edge or detects inconsistency and
        aborts.
  \item \textbf{(G3) Compact constant-round online operation.} A
        constant number of online sub-rounds per training round with no
        multi-round dispute resolution, at overall EC-tier communication cost
        $O(Q^2 d\log p)$.
\end{itemize}

\subsection{Cryptographic Setup}

\paragraph*{Pairwise OTP masks}
For each round $\tau$ and edge pair $(q,q')$ with $q<q'$, a fresh value
$r_{qq'}^{(\tau)}\xleftarrow{\$}\mathbb{F}_p$ is shared privately
between the pair, and each edge derives
\begin{equation}
  \mu_q^{(\tau)} = \sum_{q'>q} r_{qq'}^{(\tau)}
                   - \sum_{q'<q} r_{q'q}^{(\tau)} \pmod{p},
  \label{eq:otp_mask}
\end{equation}
so that $\sum_q \mu_q^{(\tau)}=0$ and each $\mu_q^{(\tau)}$ is uniform
over $\mathbb{F}_p$.
Since $Z(\tau)\in[-\sum_q D_q,+\sum_q D_q]$, we choose a prime
$p>2\sum_q D_q$ and recover the centered integer from
$Z(\tau)\bmod p$ in the standard way.

\paragraph*{MDS structure and thresholds}
Writing $t$ for $t_{\mathrm{EC}}$ throughout this section, a
degree-$t$ Shamir sharing among $Q$ edges is a codeword of a
$[Q,t+1,Q-t]$ MDS code, so unique decoding handles $e$ errors and $s$
erasures whenever $2e+s\le Q-t-1$.
Correcting the $t$ malicious edges alone requires only $t<Q/3$, but we
adopt the stricter basic threshold $t<Q/4$ so that the cloud tolerates
up to $t$ malicious forwarding errors \emph{and} up to $t$ additional
honest dropouts in the same round.
The threshold relaxes to $t<Q/3$ with the offline complaint
of~\cite{damgard2007}.
The basic threshold yields $t_{\mathrm{EC}}=0$ for $Q\le 4$ and
$t_{\mathrm{EC}}=1$ for $5\le Q\le 7$, while the complaint variant
gives $t_{\mathrm{EC}}=1$ already at $Q=4$, so very small deployments
retain per-edge output privacy (which requires only $Q\ge 2$) and
detection-with-abort, while malicious-edge correction becomes
meaningful from $Q\ge 5$.
This is intrinsic to recovery from a $[Q,t+1,Q-t]$ code rather than a
defect of the design, and Table~\ref{tab:comm_thresh} quantifies the
tolerance per evaluated topology.
Note that $t_{\mathrm{EC}}<Q/4$ is a share-consistency recovery
threshold, not a bound on global learning correctness.
Honest majority is intrinsic to signSGD-MV itself, since a malicious
majority renders the vote meaningless regardless of the aggregation
layer, and whether the weighted global vote flips is governed by the
margin condition~\eqref{eq:f4_bound} rather than by the edge count.

\subsection{Design Justification via Strawman Analysis}
\label{sec:strawman}

The EC-tier contribution is a problem-and-design-space result over
standard tools (OTP masking, Shamir sharing, and MDS decoding) rather
than a new primitive.
No flat secure-aggregation construction can be ported to this tier,
since a flat protocol treats its reconstructor as an honest output
sink and never faces the problem of hiding each edge's output
\emph{from} the reconstructor while recovering its consistency against
a misbehaving dealer.
We show that, among the natural combinations of pairwise OTP masking
and Shamir sharing considered below, every candidate fails at least one
of G1--G3, while the combination used here meets all three.
We do not claim exhaustiveness over the entire design space, but the
analysis shows the design is \emph{forced} within this space rather
than merely chosen.
Per-share information-theoretic authentication is not a substitute
within this space, since a MAC generated by the dealer binds each
share in transit but certifies nothing about cross-recipient
degree-$t$ consistency, which is exactly the dealer-layer problem
here.

\paragraph*{Candidate~A (pairwise OTP masking only)}
Each edge uploads a single masked scalar $y_q$ to the cloud, a direct
EC-tier analogue of~\cite{bonawitz2017practical}.
G1 and G3 hold, but \emph{G2 fails}: nothing binds the single
submitted value to a self-consistent report.

\paragraph*{Candidate~B (Shamir VSS on the unmasked output)}
Edge $q$ Shamir-shares $\hat{g}_q(\tau)$ directly.
G2 holds, but \emph{G1 fails} because the cloud reconstructs
$\hat{g}_q(\tau)$, and classical VSS additionally requires the
commitments we avoid.

\paragraph*{Candidate~C (masked upload then VSS audit)}
A masked upload in round 1 is followed by a Shamir audit of the same
value in round 2.
G1 and a form of G2 hold, but \emph{G3 fails} due to the extra online
round and the dispute-resolution mechanism for mismatches.

\paragraph*{Candidate~D (share-then-mask)}
Edge $q$ Shamir-shares $\hat{g}_q(\tau)$ and adds the mask outside the
share domain.
\emph{(D1) Cloud-side masking} requires the cloud to know
$\mu_q^{(\tau)}$, breaking G1.
\emph{(D2) Share-domain masking by edges} is, for an honest edge,
mathematically identical to Shamir-sharing
$y_q^{(\tau)}=D_q\hat{g}_q(\tau)+\mu_q^{(\tau)}$ directly by linearity,
i.e., our design.
This equivalence is an honest-execution statement only.
A malicious edge may add the mask to evaluations not lying on a single
degree-$t$ polynomial, but this is exactly Case~1 equivocation
(Section~\ref{sec:fabrication}), detected by the consistency check
except with soundness error $\le 2^{-\kappa}$, where $\kappa$ denotes the statistical security parameter, and corrected by MDS
decoding within the radius.
Hence D2 contains no viable design distinct from ours that satisfies G1.

\paragraph*{The surviving design}
The only design in this space satisfying G1--G3 without commitments is:
(i)~form $y_q^{(\tau)}=D_q\hat{g}_q(\tau)+\mu_q^{(\tau)}$ privately
inside edge $q$,
(ii)~Shamir-share $y_q^{(\tau)}$ \emph{itself} as the secret, and
(iii)~apply MDS decoding at the cloud to recover $y_q^{(\tau)}$
directly, never reconstructing $\hat{g}_q(\tau)$.
Privacy and recovery are overlaid on a \emph{single} object, the
OTP-masked output as the Shamir secret, rather than as two separate
steps, since separating them breaks one of G1--G3.
This combination, in the asymmetric topology of $Q$
edges-as-MPC-parties sharing to a single cloud reconstructor with
weighted aggregation, has not, to our knowledge, been studied in the
HFL signSGD-MV setting.
Table~\ref{tab:strawman_ec} summarizes.

\begin{table}[t]
\centering
\caption{EC-tier strawman comparison.}
\label{tab:strawman_ec}
\renewcommand{\arraystretch}{1.05}
\setlength{\tabcolsep}{2pt}
\scriptsize
\resizebox{\columnwidth}{!}{%
\begin{tabular}{@{}L{2.6cm}C{0.45cm}C{0.45cm}C{0.45cm}L{2.3cm}@{}}
\toprule
\textbf{Design} & \textbf{G1} & \textbf{G2} & \textbf{G3} & \textbf{Failure mode} \\
\midrule
A: OTP mask only & $\checkmark$ & $\times$ & $\checkmark$ & No binding to $y_q$ \\
B: Shamir on $\hat{g}_q$ & $\times$ & $\checkmark$ & $\checkmark$ & Cloud sees output \\
C: Mask + VSS audit & $\checkmark$ & $\checkmark$ & $\times$ & Extra rounds \\
D1: Cloud-side mask & $\times$ & $\checkmark$ & $\checkmark$ & Cloud needs mask \\
D2: Share-then-mask & \multicolumn{3}{c}{$\equiv$ ours (honest)} & Malicious deviation \\
\textbf{Ours: share $y_q$} & $\checkmark$ & $\checkmark$ & $\checkmark$ & None \\
\bottomrule
\end{tabular}%
}
\end{table}

\subsection{Protocol Description}
\label{sec:ec_construction}

We assume $Q\ge 2$, since OTP masking provides no per-edge privacy when
$Q=1$.
Algorithm~\ref{alg:ec_ext} states the protocol.

\begin{algorithm}[t]
\caption{EC-Tier Protocol with Consistency Check and MDS Recovery}
\label{alg:ec_ext}
\begin{algorithmic}[1]
\Require Edge outputs $\{\hat{g}_q(\tau)\}_{q=1}^{Q}$ from $\mathcal{F}_{\mathrm{DE}}$, weights $\{D_q\}$, threshold $t<Q/4$, prime $\mathbb{F}_p$, challenge space $\mathbb{G}\supseteq\mathbb{F}_p$ (a field extension), $|\mathbb{G}|\ge Q\cdot 2^\kappa$
\Ensure $Z(\tau)$ and $\tilde{g}^{\mathrm{glob}}(\tau)$

\Statex \textit{Masked Shamir distribution}
\ForAll{$q\in\{1,\ldots,Q\}$}
  \State Compute $\mu_q^{(\tau)}$ via \eqref{eq:otp_mask}
  \State $y_q^{(\tau)}\gets D_q\hat{g}_q(\tau)+\mu_q^{(\tau)}\pmod p$
  \State Sample degree-$t$ polynomial $f_q^{(\tau)}$, $f_q^{(\tau)}(0)=y_q^{(\tau)}$
  \State Send $f_q^{(\tau)}(v)$ to each edge $v\in\{1,\ldots,Q\}$
\EndFor

\Statex \textit{Online consistency check}
\State Cloud samples and publishes $\chi\leftarrow\mathbb{G}$ after all distributions complete
\ForAll{$v\in\{1,\ldots,Q\}$}
  \State $C_v\gets\sum_{q=1}^{Q}\chi^q f_q^{(\tau)}(v)$ and broadcast $C_v$
\EndFor
\If{$\{(v,C_v)\}$ is not a degree-$t$ codeword}
  \State \Return \textsc{Abort}
\EndIf

\Statex \textit{Cloud recovery}
\ForAll{$v$} \State Forward $\{f_q^{(\tau)}(v)\}_{q=1}^Q$ to the cloud \EndFor
\ForAll{$q$}
  \State $\bar y_q^{(\tau)}\gets\mathrm{Decode}_{\mathrm{MDS}}\bigl(\{(v,f_q^{(\tau)}(v))\},t\bigr)$
\EndFor
\State $Z(\tau)\gets\sum_q \bar y_q^{(\tau)}\pmod p$ with centered recovery
\State \Return $Z(\tau)$ and $\tilde{g}^{\mathrm{glob}}(\tau)=\mathrm{sign}(Z(\tau))$
\end{algorithmic}
\end{algorithm}

\paragraph*{Challenge generation}
The public challenge $\chi$ is sampled and published by the
semi-honest cloud only after every edge has completed its Shamir
distribution, so all distributed shares are fixed before $\chi$
becomes known and the soundness argument of
Appendix~\ref{app:ec_proof} applies against rushing edges.
Because the cloud is semi-honest and non-colluding, $\chi$ is honestly
uniform, and no inter-edge coin-tossing or commitment-based setup is
needed, preserving the commitment-free property.
Alternatively, $\chi$ can be derived by hashing the round transcript
in the random-oracle model, removing even this single cloud message.

\paragraph*{Choice of the challenge space and coordinate batching}
The protocol field $\mathbb{F}_p$ is fixed by the aggregation range
($p>2\sum_q D_q$, e.g., $p=61$ here), so sampling $\chi$ from
$\mathbb{F}_p$ itself would give soundness error only $Q/p$, which is
useless for small $p$.
The challenge is therefore drawn from an extension field
$\mathbb{G}=\mathbb{F}_{p^m}\supseteq\mathbb{F}_p$, in which the
shares embed and every polynomial identity over $\mathbb{F}_p$
remains valid, so the soundness error $Q/|\mathbb{G}|$ is driven to
$2^{-\kappa}$ by increasing $m$ alone, decoupled from the share size
$\log_2 p$.
For the $d$-dimensional update, the per-coordinate combinations are
further batched into a single check by assigning consecutive powers
of $\chi$ across (dealer, coordinate) pairs, so each edge broadcasts
one element of $\mathbb{G}$ per round.
The combined error polynomial then has degree at most $Qd$, and
choosing $|\mathbb{G}|\ge Qd\cdot 2^{\kappa}$, a single element of
roughly $\kappa+\log_2(Qd)$ bits, preserves soundness $2^{-\kappa}$
(Appendix~\ref{app:ec_proof}).

\paragraph*{Detection vs.\ correction}
The two mechanisms address disjoint layers.
\emph{(a) Dealer-layer well-formedness.}
If a malicious dealer distributes evaluations lying on no degree-$t$
polynomial, the challenge $C_v=\sum_q\chi^q f_q^{(\tau)}(v)$ inherits
the malformedness as a polynomial-in-$\chi$ relation and the joint
degree-$t$ test fails except with soundness error
$\le Q/|\mathbb{G}|\le 2^{-\kappa}$
(Appendix~\ref{app:ec_proof}), causing \textsc{Abort}.
An edge broadcasting an incorrect $C_v$ falls into the same case.
\emph{(b) Forwarding-layer MDS correction.}
If all dealers are well-formed but malicious edges corrupt shares
during forwarding (or shares are lost), MDS decoding corrects any
combination with $2e+s\le Q-t-1$ deterministically.
Recovery in this paper therefore means recovery from forwarding-layer
errors after accepted well-formed sharings, whereas malformed dealer
behavior is handled by detection and abort. In both cases the accept-or-abort decision is made and announced by
the semi-honest cloud, which collects the $C_v$ values, so no
broadcast channel or Byzantine agreement among edges is required and
the constant-round operation of G3 is preserved.

\paragraph*{Dealer identification and liveness}
The joint test certifies the existence of a malformed sharing without
identifying its dealer, so the basic abort policy would let a
persistent equivocator stall training, as the $m=3$ rows of
Table~\ref{tab:attack_eval_main} illustrate.
A natural fallback preserves liveness at $Q$ times the check cost and
runs only upon failure: the cloud applies per-dealer MDS decoding to
each forwarded sharing individually.
Under $t<Q/4$, an honest dealer's sharing always lies within the
unique-decoding radius even after $t$ forwarding errors and $t$
dropouts, so a per-dealer decoding failure certifies misbehavior by
that dealer, whose report is then excluded from $Z(\tau)$ as a dropout
and whose edge is flagged for removal in subsequent rounds.
Malicious forwarders cannot frame an honest dealer within the
threshold, and a malformed sharing that lies within the radius of some
codeword decodes to a unique well-defined value, which is exactly the
accepted forwarding-error behavior.
Persistent equivocation is thereby converted from repeated aborts into
one-time identification and exclusion.
The computational cost of the fallback equals the per-dealer MDS decoding already benchmarked as the fault path in Table~\ref{tab:compute_bench}, incurred only in rounds where the joint test fails and only once for a persistent equivocator. The experiments in Section~\ref{sec:experiments} report the basic policy without this fallback.

\subsection{Security Guarantees}

\begin{definition}[Per-edge output privacy beyond the aggregate]
\label{def:ec_privacy}
The EC-tier protocol achieves per-edge output privacy beyond the
aggregate if for any two output vectors
$\hat{\mathbf{g}},\hat{\mathbf{g}}'\in\{-1,+1\}^Q$ with
$\sum_q D_q\hat{g}_q=\sum_q D_q\hat{g}'_q$, the cloud's view under
$\hat{\mathbf{g}}$ is perfectly indistinguishable from its view under
$\hat{\mathbf{g}}'$.
Equivalently, the cloud's posterior on $\hat{\mathbf{g}}$ given its
full view equals its posterior given only the aggregate $Z$.
When up to $t$ edges are corrupted, the definition applies with the
quantification restricted to the honest edges' outputs and to equal
honest weighted sums.
\end{definition}

\begin{definition}[Share-consistency recovery]
\label{def:ec_binding}
For any strategy of at most $t<Q/4$ malicious edges with
$|\mathbb{G}|\ge Q\cdot 2^\kappa$ ($Qd\cdot 2^{\kappa}$ under
coordinate batching):
(i)~the consistency check detects, except with soundness error
$\le 2^{-\kappa}$, any sharing not lying on a degree-$t$ polynomial
(\textsc{Abort}), and
(ii)~when the check accepts, MDS decoding uniquely recovers
$y_q^{(\tau)}=f_q^{(\tau)}(0)$ within $2e+s\le Q-t-1$.
\end{definition}

\begin{remark}[Inherent aggregate leakage]
\label{rem:inherent_leak}
The aggregate $Z(\tau)$ is the intended output and inevitably reveals
some information about $\{\hat{g}_q(\tau)\}$ (e.g., $Q=2$,
$D_1=D_2=1$, $Z=+2$ implies $\hat{g}_1=\hat{g}_2=+1$).
This leakage is shared by all signSGD-MV baselines, and for small-$Q$
deployments output perturbation (e.g., DP noise at the cloud) composes
on top of \textsc{TierSecAgg} without affecting its internal analysis.
\end{remark}

\begin{lemma}[Joint distribution of pairwise OTP masks]
\label{lem:joint_otp}
Let $\{r_{qq'}^{(\tau)}\}_{q<q'}$ be independent and uniform over
$\mathbb{F}_p$ and unknown to the cloud, and define
$\mu_q^{(\tau)}$ as in~\eqref{eq:otp_mask}.
Then the joint distribution of
$\boldsymbol{\mu}^{(\tau)}=(\mu_1^{(\tau)},\ldots,\mu_Q^{(\tau)})$
from the cloud's perspective is uniform over the zero-sum hyperplane
$\mathcal{H}_0=\{\boldsymbol{\mu}\in\mathbb{F}_p^Q:\sum_q\mu_q=0\}$.
\end{lemma}

\begin{proof}[Proof sketch]
The map $\{r_{qq'}\}\mapsto\boldsymbol{\mu}$ is
$\mathbb{F}_p$-linear and surjective onto $\mathcal{H}_0$ (cf.\ the
analysis in~\cite{bonawitz2017practical}), and the image of a product
uniform distribution under a surjective linear map onto a subspace is
uniform on that subspace.
\end{proof}

\begin{proposition}[EC-tier security]
\label{prop:ec_security}
Under $t<Q/4$, the consistency check with
$|\mathbb{G}|\ge Q\cdot 2^\kappa$ ($Qd\cdot 2^{\kappa}$ under
coordinate batching), and the cloud--edge non-collusion assumption, for any rushing adversary statically corrupting at most $t$
edges, the EC-tier protocol simultaneously and independently achieves
(i)~share-consistency recovery (Definition~\ref{def:ec_binding}) and
(ii)~information-theoretic per-edge output privacy
(Definition~\ref{def:ec_privacy}, restricted to the honest edges'
outputs under corruption).
\end{proposition}
\begin{proof}
See Appendix~\ref{app:ec_proof}.
\end{proof}

\subsection{Fabrication Taxonomy and Partial Defense}
\label{sec:fabrication}

Malicious edges can deviate in four operationally distinct ways, of
which we defend three and explicitly bound the fourth.
Edge $q$ is the deviating party, honest edges follow the protocol, and the
cloud is honest-but-curious.

\paragraph*{Case~1: equivocation in Shamir distribution}
Edge $q$ sends $\{f_q^{(\tau)}(v)\}$ lying on no degree-$t$ polynomial.
\textit{Detected} by the consistency check except with soundness
error $\le 2^{-\kappa}$, causing abort
(Proposition~\ref{prop:ec_security}(i)), while forwarding-layer
inconsistency in well-formed sharings is corrected by MDS decoding
within $2e+s\le Q-t-1$.

\paragraph*{Case~2: channel-split inconsistency}
Edge $q$ tries to commit to different values on cloud-facing and
inter-edge channels.
\textit{Defended structurally}: in Algorithm~\ref{alg:ec_ext} the cloud
has no channel from edge $q$ outside the inter-edge Shamir
distribution, so no surface exists.

\paragraph*{Case~3: late repudiation}
Edge $q$ later denies having sent the recovered $y_q^{(\tau)}$.
Each edge forwards its received share over an authenticated channel,
and under $t<Q/4$ any $t+1$ of the $Q-t>3t$ honest forwards uniquely
determine $y_q^{(\tau)}$.
This is forensic reproducibility rather than cryptographic
non-repudiation, which would require signatures or commitments.

\paragraph*{Case~4: semantic fabrication}
Edge $q$ honestly Shamir-shares a well-formed but fabricated
$y_q'=D_q\hat{g}_q^{(\mathrm{fab})}+\mu_q^{(\tau)}$.
\textit{Not detected and not reduced at the EC tier}, since the sharing
is well-formed and OTP masking hides DE-tier inputs from the cloud.
The case subsumes active DE-tier deviation by the same physical edge
that alters its own reported cluster output, as noted in
Section~\ref{sec:threat}.
Under $t<Q/4$, each malicious edge contributes weighted error at most
$2D_q$, giving the threshold-induced bound
\begin{equation}
  \bigl|Z^{\mathrm{rec}}(\tau)-Z^{\mathrm{honest}}(\tau)\bigr|
  \le 2\sum_{q\in\mathcal{T}}D_q
  \le 2tD_{\max},
  \label{eq:f4_bound}
\end{equation}
with $\mathcal{T}$ the malicious set and $D_{\max}=\max_q D_q$.
The global majority flips only when
$|Z^{\mathrm{honest}}(\tau)|<2tD_{\max}$, i.e., near a tie.
Detecting semantic fabrication requires an output-validity mechanism
(e.g., zero-knowledge proofs over the vote computation or device-level
attestation) compatible with DE-tier privacy, a substantial extension
left to future work.

\section{Compositional Security}
\label{sec:composition}

The two tiers interact only through the single scalar
$\tilde{g}_q(\tau)\in\{-1,+1\}$ per cluster, an interface that is
one-directional and information-minimal, so the tiers can be analyzed
independently and composed.

\begin{theorem}[Compositional security of \textsc{TierSecAgg}]
\label{thm:composition}
Assume the DE tier instantiates any protocol realizing
$\mathcal{F}_{\mathrm{DE}}$ with $t_{\mathrm{DE}}<n_q^\tau/2$ and
recovery condition $2e+s\le d_{\min}^{\mathrm{DE}}-1$, and the EC-tier
consistency check accepts.
Under $t_{\mathrm{EC}}<Q/4$ (or $t_{\mathrm{EC}}<Q/3$ with offline
complaint~\cite{damgard2007}), the full protocol simultaneously
achieves:
\begin{enumerate}[label=(\roman*),leftmargin=*,itemsep=1pt]
  \item \textbf{DE-tier $t_{\mathrm{DE}}$-privacy}: per cluster, the
        view of any $t_{\mathrm{DE}}$-coalition is simulatable given
        $\tilde{g}_q(\tau)$ and the coalition's own inputs and
        randomness, statistically under an information-theoretic DE
        realization and computationally under PRSS
        (Remark~\ref{rem:prss}).
  \item \textbf{EC-tier per-edge output privacy beyond the aggregate}:
        for any $\hat{\mathbf{g}},\hat{\mathbf{g}}'$ inducing the same
        released aggregate, the cloud's EC-tier views are perfectly
        indistinguishable, restricted to the honest edges' outputs
        under corruption (Proposition~\ref{prop:ec_security}(ii)).
  \item \textbf{EC-tier share-consistency recovery}: conditioned on
        acceptance (which fails to detect a malformed sharing only with
        probability $\le 2^{-\kappa}$), MDS decoding uniquely recovers
        every $y_q^{(\tau)}$ against $e$ forwarding errors and $s$
        dropouts with $2e+s\le Q-t_{\mathrm{EC}}-1$
        (Proposition~\ref{prop:ec_security}(i)).
\end{enumerate}
The full simulator composes $\mathcal{S}_{\mathrm{DE}}$ and
$\mathcal{S}_{\mathrm{EC}}$ sequentially.
\end{theorem}

\begin{corollary}[End-to-end device-input privacy]
\label{cor:e2e}
Under Theorem~\ref{thm:composition}, the cloud's view reveals no
information about honest devices' signs $\{x_i(\tau)\}$ or honest
edges' unmasked outputs $\{\hat{g}_q(\tau)\}$ beyond what is implied
by the released aggregate $Z(\tau)$.
\end{corollary}

\begin{proof}[Proof sketch]
By P2 of $\mathcal{F}_{\mathrm{DE}}$ the DE-tier transcript is
simulatable from $\tilde{g}_q(\tau)$, by
Proposition~\ref{prop:ec_security}(ii) the EC-tier view is invariant
across same-aggregate output vectors, and the EC-tier transcript
contains no DE-tier randomness, so the two simulators compose
sequentially.
\end{proof}

\begin{corollary}[Concrete end-to-end guarantee]
\label{cor:concrete_e2e}
With the DE tier of Section~\ref{sec:de_concrete} and the EC tier of
Algorithm~\ref{alg:ec_ext}, under $t_{\mathrm{DE}}<n_q^\tau/2$ and
$t_{\mathrm{EC}}<Q/4$, \textsc{TierSecAgg} achieves
(a)~end-to-end device-input privacy beyond $Z(\tau)$,
(b)~DE-tier recovery with
$2e_{\mathrm{DE}}+s_{\mathrm{DE}}\le n_q^\tau-2t_{\mathrm{DE}}-1$
(conservative bound) and EC-tier recovery with
$2e_{\mathrm{EC}}+s_{\mathrm{EC}}\le Q-t_{\mathrm{EC}}-1$, and
(c)~per-round communication
$O(Qn_q d\log p_{\mathrm{DE}}+Q^2 d\log p)$ plus the
$O(Qn_q^2 d\log p_{\mathrm{DE}})$ explicit input-sharing term, which
drops to $O(Qn_q d\log p_{\mathrm{DE}})$ under PRSS at computational
DE-tier privacy.
\end{corollary}

\begin{remark}[Modularity]
\label{rem:modular}
The thresholds $t_{\mathrm{DE}}<n_q^\tau/2$ and $t_{\mathrm{EC}}<Q/4$
are independent, and replacing the DE-tier subroutine with any other
realization of $\mathcal{F}_{\mathrm{DE}}$ does not affect the EC-tier
analysis.
\end{remark}

\section{Experimental Evaluation}
\label{sec:experiments}

Under honest execution, plaintext signSGD-MV-HFL and
\textsc{TierSecAgg} reconstruct the same global sign update by
construction, so their utility results are reported jointly as
``signSGD-MV/\textsc{TierSecAgg}.''
Cryptographic claims (Propositions~\ref{prop:de_security}
and~\ref{prop:ec_security}, Theorem~\ref{thm:composition}) follow from
the reduction proofs and are not the target of empirical validation.
Communication and threshold trade-offs are characterized analytically
in Appendix~\ref{app:additional_exp}, together with computation-time
benchmarks of all cryptographic operations.

\subsection{Setup}
We use MNIST and FMNIST as standard benchmarks and UCI HAR (raw
$9\times128$ inertial signals, six classes) as an IoT sensor workload,
with 30 clients in total.
The MNIST/FMNIST model is a lightweight 2D CNN with two $5\times5$
convolutional layers (32 and 64 channels, each followed by $2\times2$
max pooling) and two fully connected layers of 128 and 10 units
($d=454{,}922$ parameters), and the UCI HAR model is a lightweight 1D
CNN with three convolutional layers (32, 64, and 128 channels with
kernel sizes 5, 5, and 3), global average pooling, and a linear
classifier ($d=37{,}254$ parameters).
MNIST/FMNIST use Dirichlet non-IID partitions with $\alpha=0.5$ and UCI
HAR uses the subject-level natural partition.
Five topologies
$(Q,n_q)\in\{(6,5),(5,6),(10,3),(3,10),(1,30)\}$ are evaluated with edge weights $D_q=n_q$ and ties resolved by
$\mathrm{sign}(0)=+1$.
Baselines are FedAvg (dense gradients, 5 local SGD steps) and plaintext
signSGD-MV-HFL, trained for 200 global rounds with 1 DE-tier iteration
unless noted. All reported values are the mean and standard deviation over
three random seeds.

\subsection{Utility Preservation and Topology Sensitivity}
\label{sec:exp_utility}
Table~\ref{tab:topology_signsgd_detail} reports final test accuracy,
with $6{\times}5$ as the representative reference.
At $(6,5)$, \textsc{TierSecAgg} outperforms FedAvg on MNIST
($96.84$ vs $95.54$) and FMNIST ($82.70$ vs $79.93$), while FedAvg is
stronger on UCI HAR ($93.59$ vs $89.22$), a reflection of the
dense-vs-sign algorithmic gap rather than the secure aggregation
layer.
FedAvg is nearly topology-invariant (range $<0.05$\,percentage points (pp), hence a single
column), whereas signSGD-MV/\textsc{TierSecAgg} shows moderate topology
dependence ($\le 2.1$\,pp) because majority voting is non-linear at
both tiers.
All topologies converge to $88$--$90$\,\% within 200 rounds on UCI HAR,
and at $(6,5)$ increasing the DE-tier iteration count accelerates early
convergence with saturation around $93$--$94\%$ for
$K_{\mathrm{DE}}\ge 10$, indicating that the hierarchical decomposition
does not significantly degrade learning utility.

\begin{table}[t]
\centering
\caption{Test accuracy (\%) per topology.}
\label{tab:topology_signsgd_detail}
\renewcommand{\arraystretch}{1.0}
\setlength{\tabcolsep}{1.6pt}
\scriptsize
\resizebox{\columnwidth}{!}{%
\begin{tabular}{@{}lcccccc@{}}
\toprule
\multirow{2}{*}{\textbf{Dataset}}
  & \multicolumn{5}{c}{\textbf{signSGD-MV/\textsc{TierSecAgg}}}
  & \multirow{2}{*}{\textbf{FedAvg}} \\
\cmidrule(lr){2-6}
  & $6{\times}5$ & $5{\times}6$ & $10{\times}3$ & $3{\times}10$ & $1{\times}30$ &  \\
\midrule
MNIST   & $\mathbf{96.84{\pm}0.34}$ & $96.31{\pm}0.66$ & $96.69{\pm}0.36$ & $96.47{\pm}0.09$ & $96.62{\pm}0.23$ & $95.54{\pm}0.49$ \\
FMNIST  & $\mathbf{82.70{\pm}0.77}$ & $80.64{\pm}1.34$ & $82.00{\pm}1.46$ & $81.35{\pm}1.23$ & $81.96{\pm}0.72$ & $79.93{\pm}0.21$ \\
UCI HAR & $89.22{\pm}2.16$ & $\mathbf{89.32{\pm}1.17}$ & $89.00{\pm}0.94$ & $88.61{\pm}0.59$ & $87.95{\pm}1.40$ & $93.59{\pm}0.22$ \\
\bottomrule
\end{tabular}%
}
\end{table}

\paragraph*{60-client sensitivity}
To examine whether the trend persists beyond 30 clients, we ran MNIST with $N=60$ over the large-cluster topologies $(5,12)$ and $(3,20)$ and DE-tier iteration counts $K_{\mathrm{DE}}\in\{1,5,10,20,30,50\}$. A single DE-tier iteration is clearly suboptimal, and $K_{\mathrm{DE}}\ge 5$ saturates at $98.3$--$99.1\%$ with small seed variance and no degradation up to $K_{\mathrm{DE}}=50$ (Table~\ref{tab:mnist_60_de_sensitivity}).

\begin{table}[t]
\centering
\caption{MNIST 60-client sensitivity for representative topologies.
Test accuracy (\%) over 200 rounds.}
\label{tab:mnist_60_de_sensitivity}
\renewcommand{\arraystretch}{1.05}
\setlength{\tabcolsep}{1.2pt}
\scriptsize
\begin{tabular}{@{}lcccccc@{}}
\toprule
\textbf{Topology} & \textbf{DE=1} & \textbf{DE=5} & \textbf{DE=10} & \textbf{DE=20} & \textbf{DE=30} & \textbf{DE=50} \\
\midrule
$5\times12$ & 96.63$\pm$0.63 & 98.49$\pm$0.33 & 98.74$\pm$0.01 & 98.82$\pm$0.11 & 98.81$\pm$0.20 & 98.79$\pm$0.27 \\
$3\times20$ & 96.17$\pm$0.44 & 98.27$\pm$0.21 & 98.93$\pm$0.04 & 98.85$\pm$0.23 & 99.08$\pm$0.03 & 98.86$\pm$0.21 \\
\bottomrule
\end{tabular}
\end{table}

\subsection{EC-Tier Attack Evaluation}
\label{sec:exp_attack}

We evaluate the learning-layer effect of EC-tier malicious behavior at
$(Q,n_q)=(10,3)$, the largest EC-tier party count among the evaluated
configurations ($t_{\mathrm{EC}}\le 2$ under $t<Q/4$).
This illustrates how the protocol-level distinction of
Section~\ref{sec:fabrication} appears at the learning layer, using the
injection framework of Appendix~\ref{app:attack_framework}.
\emph{Distribution-level inconsistency} (Case~1) is injected with and
without the EC-tier protection, and \emph{semantic fabrication}
(Case~4) is injected as Byzantine-flip and Byzantine-random modes,
which the consistency check accepts by design.

\begin{table}[t]
\centering
\caption{Empirical EC-tier attack evaluation on UCI HAR
($(Q,n_q)=(10,3)$, $t_{\mathrm{EC}}=2$).}
\label{tab:attack_eval_main}
\renewcommand{\arraystretch}{1.05}
\setlength{\tabcolsep}{6pt}
\scriptsize
\begin{tabular}{@{}llc@{}}
\toprule
\textbf{Mode} & \textbf{$m$} & \textbf{Test acc.\ (\%)} \\
\midrule
Clean baseline & $0$ & $91.63{\pm}1.09$ \\
\midrule
\multirow{3}{*}{Byzantine-flip}
  & $1$ & $93.58{\pm}0.31$ \\
  & $2$ & $92.25{\pm}1.12$ \\
  & $3$ & $90.06{\pm}1.25$ \\
\midrule
\multirow{3}{*}{Byzantine-random}
  & $1$ & $92.82{\pm}0.76$ \\
  & $2$ & $93.37{\pm}0.57$ \\
  & $3$ & $93.25{\pm}0.24$ \\
\midrule
\multirow{3}{*}{Equivocation (prot.)}
  & $1$ & $91.63{\pm}1.09$ \\
  & $2$ & $91.63{\pm}1.09$ \\
  & $3$ & $\phantom{0}18.90{\pm}8.30$ \\
\midrule
\multirow{3}{*}{Equivocation (unprot.)}
  & $1$ & $93.58{\pm}0.31$ \\
  & $2$ & $92.25{\pm}1.12$ \\
  & $3$ & $90.06{\pm}1.25$ \\
\bottomrule
\end{tabular}
\end{table}

Table~\ref{tab:attack_eval_main} and Fig.~\ref{fig:uci-har-ec-attack}
confirm the distinction.
In the protected mode with $m\le t_{\mathrm{EC}}=2$, all 200 attacked
rounds are detected and corrected and accuracy matches the clean
baseline ($91.63\%$).
At $m=3>t_{\mathrm{EC}}$, every round is detected but aborted rather
than corrected, so accuracy collapses to $18.90\%$ (near the 6-class
random level), the intended fail-safe of aborting instead of silently
corrupting the model under the basic policy, while the
dealer-identification fallback of Section~\ref{sec:ec_construction} would
instead exclude the equivocating edges and continue.
The unprotected mode never detects the inconsistency, and its accuracy
coincides exactly with Byzantine-flip because, absent the check, an
equivocating share reduces to a sign-flipped edge report.
Semantic fabrication remains within $\sim1.6$\,pp of baseline even at
$m=3$, consistent with~\eqref{eq:f4_bound} since UCI HAR votes are
rarely near a tie, and configurations marginally above baseline reflect
seed variance.

\begin{figure}[t]
\centering
\includegraphics[width=0.95\columnwidth]{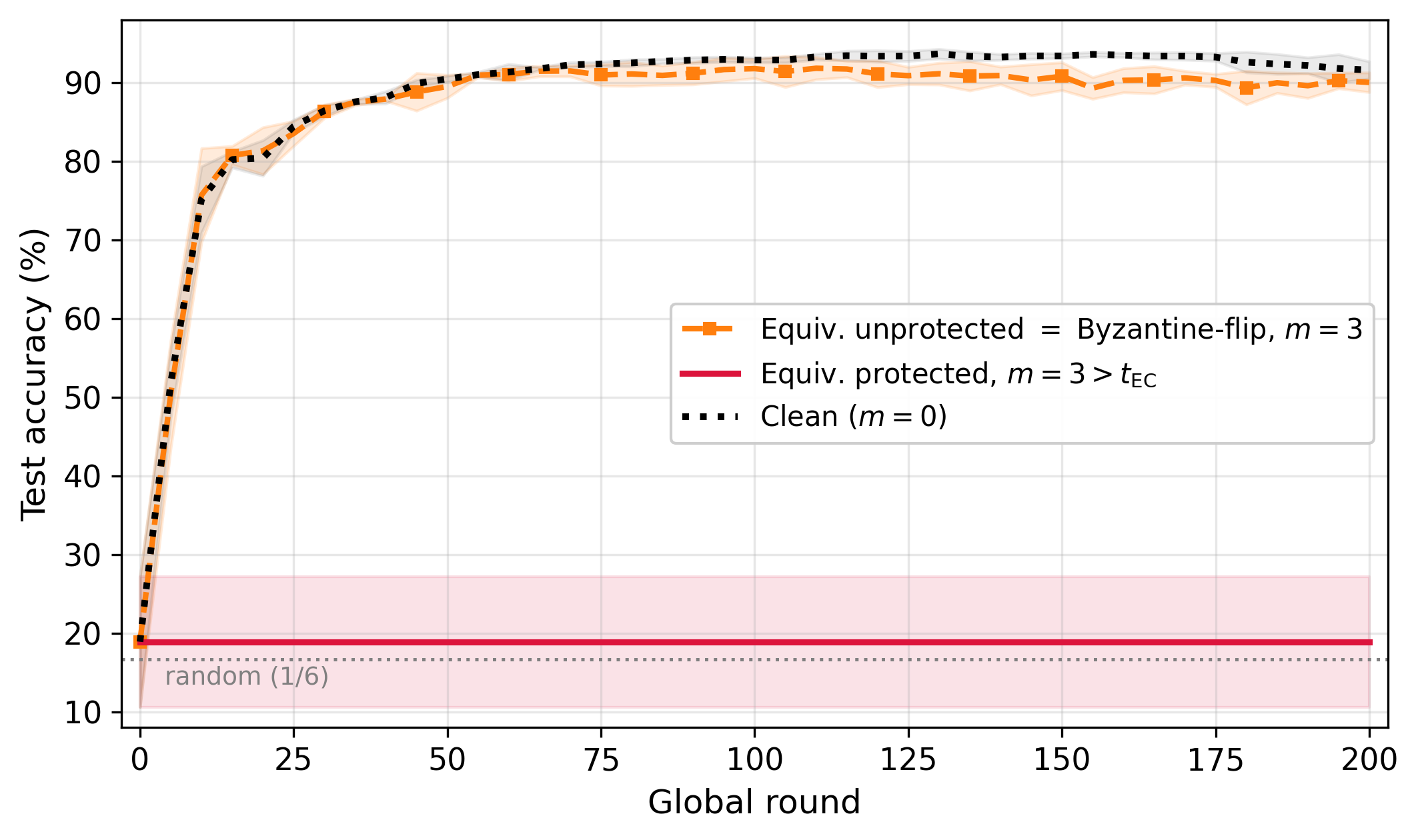}
\caption{UCI HAR convergence under representative EC-tier attacks at
$(Q,n_q)=(10,3)$, $t_{\mathrm{EC}}=2$.}
\label{fig:uci-har-ec-attack}
\end{figure}

\subsection{DE-Tier Byzantine and Dropout Recovery}
\label{sec:exp_de_recovery}

We verify the DE-tier recovery boundary of
Proposition~\ref{prop:de_security}(iii) at $(Q,n_q)=(5,6)$ with
$t_{\mathrm{DE}}=1$, giving the conservative region
$2e+s\le n_q-2t_{\mathrm{DE}}-1=3$.
Per round and cluster we inject $e$ Byzantine devices (sign-flipped or
random share forwards) and $s$ dropouts, sweeping $(e,s)$ inside and
outside the region.

\begin{table}[t]
\centering
\caption{Empirical DE-tier Byzantine/dropout recovery on UCI HAR
($(Q,n_q)=(5,6)$, $t_{\mathrm{DE}}=1$, conservative bound
$2e+s\le 3$).}
\label{tab:de_recovery_eval}
\renewcommand{\arraystretch}{1.05}
\setlength{\tabcolsep}{5pt}
\scriptsize
\begin{tabular}{@{}ccccc@{}}
\toprule
\textbf{$e$} & \textbf{$s$} & \textbf{$2e+s$}
& \textbf{Test acc.\ (\%)} & \textbf{Recovered} \\
\midrule
$0$ & $0$ & $0$ & $92.33{\pm}0.34$ & Yes \\
$0$ & $1$ & $1$ & $92.33{\pm}0.34$ & Yes \\
$1$ & $0$ & $2$ & $92.33{\pm}0.34$ & Yes \\
$0$ & $2$ & $2$ & $92.33{\pm}0.34$ & Yes \\
$1$ & $1$ & $3$ & $92.33{\pm}0.34$ & Yes \\
$0$ & $3$ & $3$ & $92.33{\pm}0.34$ & Yes \\
\midrule
$2$ & $0$ & $4$ & $90.84{\pm}0.72$ & No \\
$2$ & $1$ & $5$ & $72.14{\pm}2.32$ & No \\
$3$ & $0$ & $6$ & $40.13{\pm}16.84$ & No \\
\bottomrule
\end{tabular}
\end{table}

\begin{figure}[t]
\centering
\includegraphics[width=0.95\columnwidth]{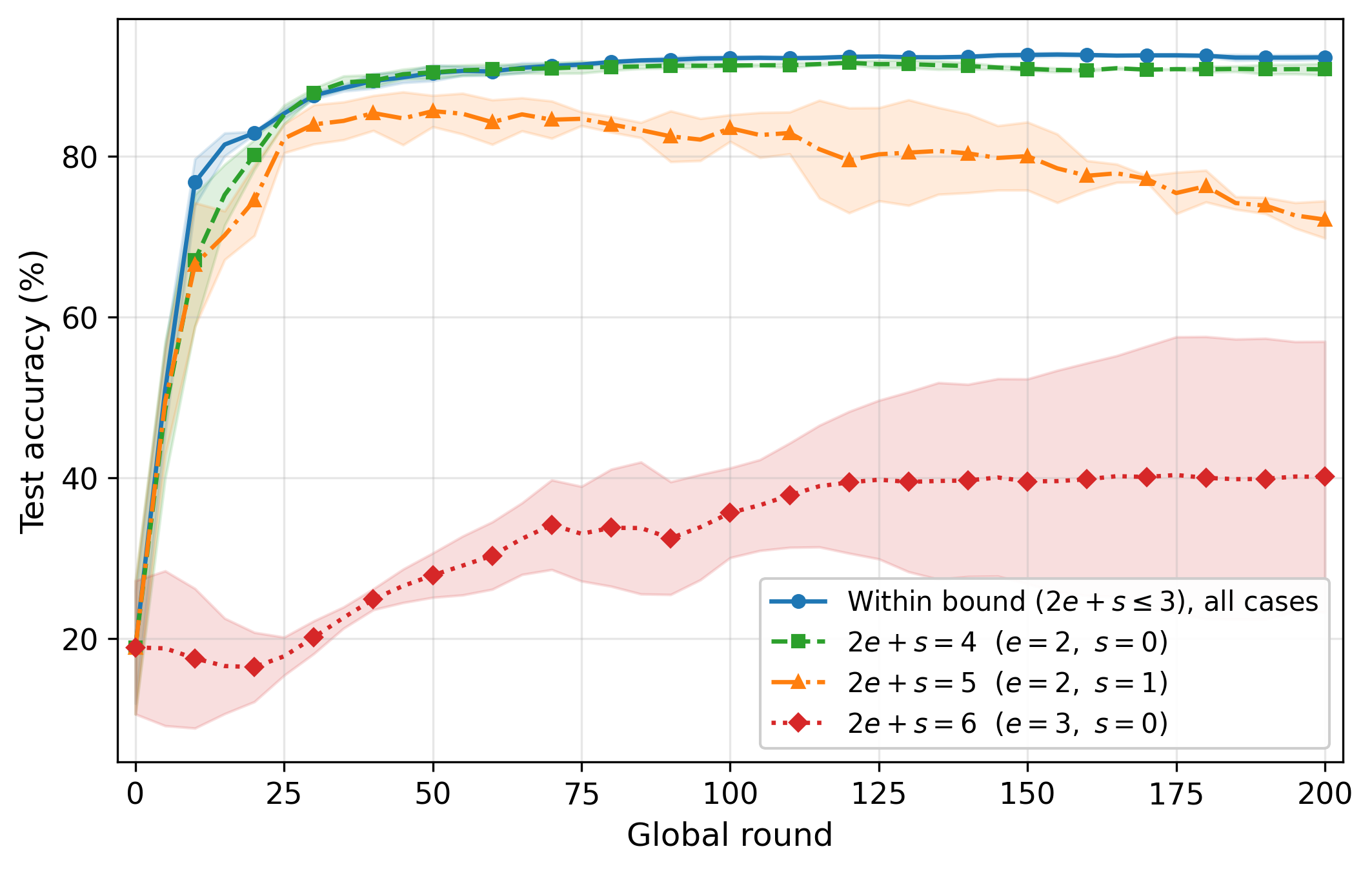}
\caption{UCI HAR convergence under DE-tier Byzantine/dropout injection
at $(Q,n_q)=(5,6)$, $t_{\mathrm{DE}}=1$.
}
\label{fig:uci-har-de-recovery}
\end{figure}

Table~\ref{tab:de_recovery_eval} and
Fig.~\ref{fig:uci-har-de-recovery} confirm the boundary.
Every $(e,s)$ with $2e+s\le 3$ recovers in every round of every run
at baseline accuracy ($92.33\%$), with all six convergence curves
coinciding exactly with the clean baseline and therefore drawn as a
single curve.
Outside the bound recovery fails in every round and accuracy degrades
monotonically with the violation margin ($90.84\%$, $72.14\%$, and
$40.13\%$ at $2e+s=4,5,6$), with the large variance at $2e+s=6$
reflecting unstable sign-flipped cluster outputs.
Immediately beyond the boundary the degradation is mild, since
uncorrected errors leave most cluster signs intact and are absorbed by
the global vote, but it compounds as errors accumulate across
clusters.

The experiments confirm that
(i)~\textsc{TierSecAgg} preserves the learning utility of plaintext
signSGD-MV-HFL under honest execution,
(ii)~the topology choice exposes a DE--EC trade-off in which larger
clusters improve DE-tier feasibility while larger $Q$ improves EC-tier
malicious-edge tolerance at $O(Q^2 d\log p)$ cost, and
(iii)~distribution-level EC inconsistency is handled by the
consistency-check/MDS mechanism within threshold, while semantic
fabrication appears only as bounded degradation governed
by~\eqref{eq:f4_bound}.

\section{Conclusion}
\label{sec:discussion}

In this paper, we have presented \textsc{TierSecAgg}, a tier-stratified secure aggregation
architecture for signSGD-based HFL in IoT edge--cloud systems.
It decomposes HFL signSGD-MV secure aggregation into a DE-tier ideal
functionality and an EC-tier protocol connected through a one-directional,
information-minimal interface.
The DE tier introduces single-mask power sharing to realize hierarchical operation.
The EC tier addresses the security surface specific to HFL signSGD-MV
with the only commitment-free OTP/Shamir/MDS design among the
examined candidates that meets it.
A fabrication taxonomy separates defended behaviors from
threshold-bounded ones.
The compositional security of the full protocol is established,
jointly providing privacy, dropout recovery, and aggregation-layer
Byzantine recovery without adversary-specific verification
subprotocols.
Experiments confirm that the secure realization preserves the
plaintext learning behavior under honest execution and validate the
analytical recovery boundaries.
The tier-stratified interface allows each future extension, including
output-validity mechanisms that detect semantic fabrication, joint
DE/EC and adaptive corruption models, and sub-quadratic EC-tier
variants for larger multi-site deployments, to be pursued
independently.

\appendices

\section{Proof of Proposition~\ref{prop:de_security}}
\label{app:de_proof}

Fix a cluster $q$ and round $\tau$ with active set $S_q^\tau$ of size
$n_q^\tau$, threshold $t_{\mathrm{DE}}<n_q^\tau/2$, field
$\mathbb{F}_{p_{\mathrm{DE}}}$ with $p_{\mathrm{DE}}>n_q^\tau$, and MV
polynomial $F_q(\cdot)$ of degree $N\le p_{\mathrm{DE}}-1$.

\paragraph*{(i) Constant-round online operation}
By inspection, the online phase consists of one share upload of
$[\hat{x}^{(\tau)}]_{t_{\mathrm{DE}}}$, one broadcast of
$\hat{x}^{(\tau)}$, and one share upload of
$[F_q(x_q^{(\tau)})]_{t_{\mathrm{DE}}}$, with all preparation
input-independent and offline.

\paragraph*{(ii) $t_{\mathrm{DE}}$-privacy}
Let $\mathcal{C}\subseteq S_q^\tau\cup\{q\}$ be any semi-honest
coalition with $|\mathcal{C}|\le t_{\mathrm{DE}}$.
The coalition holds at most $t_{\mathrm{DE}}$ Shamir shares of the
fresh mask $a$, of each power $a^j$, of each coefficient $P_r$, and
of each honest input, all uniform and input-independent in its view
by Shamir perfect secrecy~\cite{shamir1979}.
Hence the sole online opening
$\hat{x}^{(\tau)}=x_q^{(\tau)}-a$ is one-time padded and therefore
uniform over $\mathbb{F}_{p_{\mathrm{DE}}}$ and independent of
$x_q^{(\tau)}$, and the post-opening step is a purely local linear
combination whose reconstruction yields only $\tilde{g}_q(\tau)$.
The simulator $\mathcal{S}_{\mathrm{DE}}$, given $\tilde{g}_q(\tau)$
and the coalition's inputs and randomness, outputs
(a)~the coalition's own state verbatim,
(b)~uniform values for the coalition's evaluations of every honest
sharing,
(c)~a uniform $\hat{x}^*\xleftarrow{\$}\mathbb{F}_{p_{\mathrm{DE}}}$
for the masked opening, and
(d)~for the final opening a uniformly random degree-$t_{\mathrm{DE}}$
polynomial matching the coalition's local linear combination on
coalition indices and reconstructing to $\tilde{g}_q(\tau)$ at $0$.
Each simulated distribution equals the corresponding real one
exactly, so the simulation is perfect.
Under the PRSS instantiation the honest sharings are PRG outputs, and a
standard hybrid argument substituting uniform sharings gives
computational privacy with total advantage
$\le\mathsf{negl}(\lambda)$ in the PRG security parameter $\lambda$.

\paragraph*{(iii) Joint dropout/Byzantine tolerance}
An opening of a degree-$\delta$ sharing among $n_q^\tau$ parties is a
codeword of an $[n_q^\tau,\delta+1,n_q^\tau-\delta]$ MDS code, so
unique decoding tolerates $2e+s\le n_q^\tau-\delta-1$.
The final opening has $\delta=t_{\mathrm{DE}}$ (setup-completed
bound), while degree-reduction openings during offline DN
multiplication~\cite{damgard2007} have $\delta=2t_{\mathrm{DE}}$,
giving the conservative bound
$2e+s\le n_q^\tau-2t_{\mathrm{DE}}-1$ when preparation runs on the
online critical path.
Under the applicable condition every opening succeeds
deterministically and the output matches the honest execution.
\hfill\qed

\section{Proof of Proposition~\ref{prop:ec_security}}
\label{app:ec_proof}

\paragraph*{(i) Share-consistency recovery}
Let $\mathcal{S}$ be the set of honest edges,
$|\mathcal{S}|\ge Q-t\ge 3t+1$, whose $C_v$ are computed faithfully,
while the adversary sees $\chi$ before broadcasting and chooses its
own $C_v$ values and forwarded shares arbitrarily.
A degree-$t$ polynomial is determined by any $t+1$ evaluations and
$|\mathcal{S}|\ge t+2$, so the joint test can accept only if the
honest restriction $\{(v,C_v)\}_{v\in\mathcal{S}}$ is itself
consistent with a single degree-$t$ polynomial: the adversary's
coordinates can be filled in to match it but never repair an
inconsistency among the honest coordinates.
Let $H_{\mathcal{S}}$ be a parity-check matrix of the code punctured
to $\mathcal{S}$, again MDS with $|\mathcal{S}|-t-1\ge 1$ parity
rows, with entries in $\mathbb{F}_p$, and remaining a parity-check
matrix of the code extended by scalars to $\mathbb{G}$.
Call dealer $q$ \emph{effectively malformed} if
$\boldsymbol{\sigma}_q=H_{\mathcal{S}}\,\boldsymbol{f}_q|_{\mathcal{S}}
\ne\boldsymbol{0}$.
A dealing inconsistent only at corrupted coordinates defines a unique
degree-$t$ polynomial and a unique $y_q^{(\tau)}$ through its honest
coordinates and is equivalent to a well-formed dealing followed by
forwarding errors.
If some dealer $q^*$ is effectively malformed, acceptance requires
$H_{\mathcal{S}}(\boldsymbol{C}|_{\mathcal{S}})
=\sum_{q=1}^{Q}\chi^q\boldsymbol{\sigma}_q=\boldsymbol{0}$.
Each coordinate of this sum is a univariate polynomial in $\chi$ over
$\mathbb{F}_p$ of degree at most $Q$ with zero constant term, at
least one coordinate is nonzero because
$\boldsymbol{\sigma}_{q^*}\ne\boldsymbol{0}$, and such a polynomial
has at most $Q$ roots in the extension field $\mathbb{G}$, so
acceptance occurs with probability at most
$Q/|\mathbb{G}|\le 2^{-\kappa}$ over the uniform
$\chi\leftarrow\mathbb{G}$, published only after all dealings are
fixed (Section~\ref{sec:ec_protocol}).
Conditioned on acceptance, every dealing agrees on the honest
coordinates with a unique degree-$t$ codeword of the $[Q,t+1,Q-t]$
code.
Deviations at corrupted coordinates, whether introduced at dealing or
at forwarding, together with honest dropouts then amount to $e\le t$
errors and $s$ erasures relative to this codeword, and under $t<Q/4$,
$e_{\max}=\lfloor(Q-t-1)/2\rfloor\ge t$, so MDS unique decoding
recovers $y_q^{(\tau)}=f_q^{(\tau)}(0)$ whenever $2e+s\le Q-t-1$.
Under coordinate batching the identical argument applies with degree
at most $Qd$, giving soundness $Qd/|\mathbb{G}|\le 2^{-\kappa}$ for
$|\mathbb{G}|\ge Qd\cdot 2^{\kappa}$.

\paragraph*{(ii) Per-edge output privacy}
Consider first the all-honest execution.
The cloud's view consists of the forwarded shares, which determine
$\mathbf{y}$ and $Z(\tau)=\sum_q y_q^{(\tau)}$, while the
consistency-check messages are deterministic functions of these
shares and the public $\chi$.
Under non-collusion the cloud sees no pairwise OTP value, so by
Lemma~\ref{lem:joint_otp} $\boldsymbol{\mu}^{(\tau)}$ is uniform over
the zero-sum hyperplane $\mathcal{H}_0$.
For any $\hat{\mathbf{g}},\hat{\mathbf{g}}'$ with the same aggregate
$Z$, the shift $\Delta:=D(\hat{\mathbf{g}}-\hat{\mathbf{g}}')$ has
zero component sum, so $\Delta\in\mathcal{H}_0$ and
$\boldsymbol{\mu}\mapsto\boldsymbol{\mu}+\Delta$ is a bijection of
$\mathcal{H}_0$, making
$\mathbf{y}=D\hat{\mathbf{g}}+\boldsymbol{\mu}$ and
$\mathbf{y}'=D\hat{\mathbf{g}}'+\boldsymbol{\mu}'$ identically
distributed by one-time-pad
uniformity~\cite{shannon1949communication}.
Given $\mathbf{y}$, each dealer samples a fresh uniform degree-$t$
polynomial with constant term $y_q^{(\tau)}$ and the cloud observes
exactly its evaluations, so the view is a fixed randomized function
of $\mathbf{y}$ and is identically distributed under
$\hat{\mathbf{g}}$ and $\hat{\mathbf{g}}'$.
Under corruption of at most $t$ edges the argument applies to the
honest coordinates.
The adversary sees at most $t$ evaluations of each honest dealing,
independent of the honest $y_q^{(\tau)}$ by Shamir perfect
secrecy~\cite{shamir1979}, plus its own inputs and OTP values, so
every adversarial message is a randomized function of a view
independent of the honest outputs.
Lemma~\ref{lem:joint_otp} applied to the pairwise OTPs among honest
edges alone makes the honest masks, conditioned on everything the
adversary knows, uniform over the zero-sum hyperplane of
$\mathbb{F}_p^{\mathcal{S}}$ up to an adversary-known offset, and the
shift argument yields perfect indistinguishability under any two
honest-output assignments with the same honest weighted sum, which is
Definition~\ref{def:ec_privacy} restricted to the honest edges.
\hfill\qed

\section{Empirical Attack Framework}
\label{app:attack_framework}

Let $\mathbf{u}_q^{(r)}\in\{-1,+1\}^d$ be edge $q$'s output at round
$r$, with the EC tier computing
$\mathbf{g}^{\mathrm{glob},(r)}=\mathrm{sign}(\sum_q D_q\widetilde{\mathbf{u}}_q^{(r)})$
over possibly corrupted $\widetilde{\mathbf{u}}_q^{(r)}$.
A malicious set of size $m$ attacks with probability
$p_{\mathrm{atk}}$ on coordinate fraction $\rho$, under modes
\emph{none}, \emph{Byzantine-flip}, \emph{Byzantine-random},
\emph{equivocation-unprotected}, and \emph{equivocation-protected}
(recover if $m\le t_{\mathrm{EC}}$, else abort, a conservative
policy that reserves the remaining decoding margin for concurrent
dropouts rather than exhausting it on error correction).
The Byzantine modes correspond to Case~4 and are bounded
by~\eqref{eq:f4_bound} rather than detected, while the protected mode
evaluates the defense against distribution-level inconsistency
(Cases~1--3) at the abstraction level of edge outputs.

\section{Communication and Threshold Analysis}
\label{app:additional_exp}

\paragraph*{Topology, threshold, and security role}
Table~\ref{tab:comm_thresh} summarizes the security role, EC-tier
tolerance, and $Q^2$ scaling factor of each evaluated topology.
For a fixed client count, increasing $Q$ improves the admissible
$t_{\mathrm{EC}}$ at quadratic inter-edge cost, while increasing
$n_q$ improves the feasibility of threshold-based DE-tier secure
majority evaluation, yielding a communication--security Pareto
frontier whose operating point depends on deployment priorities.

\begin{table}[t]
\centering
\caption{EC-tier scaling, tolerance, and security role per topology.}
\label{tab:comm_thresh}
\renewcommand{\arraystretch}{1.05}
\setlength{\tabcolsep}{2.0pt}
\scriptsize
\resizebox{\columnwidth}{!}{%
\begin{tabular}{@{}lcccL{3.2cm}@{}}
\toprule
\textbf{Topology} & $Q^2$ & $t<Q/4$ & $t<Q/3$ (compl.) & 
\textbf{Primary role} \\
\midrule
$1{\times}30$  & N/A & N/A & N/A & 
Flat baseline, no EC-tier surface. \\
$3{\times}10$  & 9   & $=0$ & $=0$ & 
Large DE clusters, no EC malicious-edge recovery. \\
$5{\times}6$   & 25  & $\le 1$ & $\le 1$ & 
Balanced operating point. \\
$6{\times}5$   & 36  & $\le 1$ & $\le 1$ & 
Balanced, representative setting. \\
$10{\times}3$  & 100 & $\le 2$ & $\le 3$ & 
EC robustness stress point, small clusters limit DE tolerance. \\
\bottomrule
\end{tabular}%
}
\end{table}

\paragraph*{Scaling comparison}
Table~\ref{tab:comm_scaling} compares per-round communication scaling.
\textsc{TierSecAgg} matches additive SecAgg at the DE tier and pays an
additional $Q$ factor at the EC tier in exchange for per-edge output
privacy and share-consistency recovery, an overhead intrinsic to
EC-tier accountability without commitments.

\begin{table}[t]
\centering
\caption{Per-round communication scaling (bits).}
\label{tab:comm_scaling}
\renewcommand{\arraystretch}{1.0}
\setlength{\tabcolsep}{2.0pt}
\scriptsize
\resizebox{\columnwidth}{!}{%
\begin{tabular}{@{}lccc@{}}
\toprule
\textbf{Method} & \textbf{DE tier (bits)} & \textbf{EC tier (bits)} & \textbf{EC consistency} \\
\midrule
FedAvg                & $O(32\,Qn_qd)$ & $O(32\,Qd)$   & No \\
Plain signSGD-MV      & $O(Qn_qd)$ & $O(Qd)$    & No \\
Additive SecAgg       & $O(32\,Qn_qd)$ & $O(32\,Qd)$   & No \\
HE aggregation        & $O(Qn_qd\log \lambda_{\mathrm{HE}})$ & $O(Qd\log \lambda_{\mathrm{HE}})$ & Limited \\
\textsc{TierSecAgg}   & $O(Qn_qd\log p_{\mathrm{DE}})^{\ast}$ & $O(Q^2 d\log p)$ & Yes \\
\bottomrule
\multicolumn{4}{@{}l}{\scriptsize
$^{\ast}$MV-evaluation layer. Default explicit input sharing adds
$O(Qn_q^2 d\log p_{\mathrm{DE}})$ (Section~\ref{sec:de_concrete}).}
\end{tabular}%
}
\end{table}

\paragraph*{Concrete communication costs}
For the 30-client deployments ($Qn_q=30$ and $D_q=n_q$, so
$\sum_q D_q=30$), the
minimum-size primes are $p_{\mathrm{DE}}=31$
($\approx 5$\,bits) and $p=61$ ($\approx 6$\,bits), and under
coordinate batching the consistency check contributes only one
element of $\mathbb{G}$ per edge per round, roughly
$\kappa+\log_2(Qd)$ bits, negligible next to the share traffic.
The DE-tier online cost is
$\approx 2Qn_q d\log_2 p_{\mathrm{DE}}/8$ bytes and the EC-tier cost
is $\approx(2Q^2-Q)d\log_2 p/8$ bytes.
Table~\ref{tab:concrete_mb} reports the resulting per-round costs at
the representative $(6,5)$: on UCI HAR ($d=37{,}254$),
\textsc{TierSecAgg} incurs $1.40$\,MB (DE, MV-evaluation) and
$1.84$\,MB (EC), below the FedAvg cost of $4.47$\,MB and well within
IoT uplink budgets.
Across topologies, the explicit input-sharing cost grows with cluster
size (up to $264.42$\,MB at $(1,30)$ for the MNIST/FMNIST model),
whereas the PRSS instantiation is topology-independent ($25.59$\,MB
for MNIST/FMNIST, $2.10$\,MB for UCI HAR), a saving that grows from
$2.0\times$ at
$(6,5)$ to $10.3\times$ at $(1,30)$ at the cost of computational
rather than information-theoretic DE-tier privacy.
The EC-tier cost grows with $Q^2$, reaching $5.31$\,MB (UCI HAR) and
$64.83$\,MB (MNIST/FMNIST, $d=454{,}922$) at $(10,3)$.
Offline preparation is input-independent with $O(N)$ storage per cluster, but its communication is not negligible. With $2n_q$ field elements exchanged per DN multiplication, the per-round offline traffic at $(6,5)$ with $N=30$ is approximately $40.5$\,MB for UCI HAR and $494.7$\,MB for the MNIST/FMNIST model, and
explicit batch generation of the double sharings adds a further $54.0$\,MB and $659.6$\,MB respectively, a term that vanishes under the PRSS instantiation of Remark~\ref{rem:prss}. Offline preparation runs in parallel across clusters on their respective edges and overlaps with local training, and Table~\ref{tab:compute_bench} reports the measured per-cluster preparation time.

\begin{table}[t]
\centering
\caption{Approximate per-round communication (MB) at $(Q,n_q)=(6,5)$.}
\label{tab:concrete_mb}
\renewcommand{\arraystretch}{1.0}
\setlength{\tabcolsep}{2.0pt}
\scriptsize
\resizebox{\columnwidth}{!}{%
\begin{tabular}{@{}lccccccc@{}}
\toprule
\textbf{Model} & \textbf{$d$} & \textbf{Plain DE}
& \textbf{Secure DE (MV)} & \textbf{Secure DE (full)}
& \textbf{Secure DE (PRSS)}
& \textbf{FedAvg DE} & \textbf{EC ($6{\times}5$)} \\
& & (1 bit) & ($\log p_{\mathrm{DE}}$) & ($\log p_{\mathrm{DE}}$)
& ($\log p_{\mathrm{DE}}$)
& (32-bit float) & ($\log p$) \\
\midrule
MNIST/FMNIST CNN & $454{,}922$ & $1.71$ & $17.06$ & $51.18$ & $25.59$ & $54.59$ & $22.52$ \\
UCI HAR 1D-CNN   & $37{,}254$  & $0.14$ & $1.40$  & $4.19$  & $2.10$  & $4.47$  & $1.84$ \\
\bottomrule
\end{tabular}%
}
\end{table}

\paragraph*{Comparison with other secure aggregation methods}
Any secret-sharing scheme needs shares over a field larger than the
number of share-holders, so any sign-based secure aggregation with
$n_q\le 30$ costs at least
$\log_2 p_{\mathrm{DE}}\approx 5$\,bits/coord, and
\textsc{TierSecAgg}'s $30$\,bits/coord is exactly six such elements
per device per coordinate (input sharing to the $n_q-1=4$ peers plus
two online MV-evaluation uploads), while plain signSGD-MV's
$1$\,bit/coord is unattainable by any secure realization.
Table~\ref{tab:secure_agg_compare} compares secure methods at $(6,5)$
on UCI HAR: \textsc{TierSecAgg} is $\approx68\times$ cheaper than
HE-based FL ($4.19$ vs $\approx286$\,MB at
$\log_2\lambda_{\mathrm{HE}}=2{,}048$, a gap of $\sim3.5$\,GB/round
on the MNIST/FMNIST model) and within $\sim6\%$ of additive SecAgg
while additionally providing HFL signSGD-MV compatibility, EC-tier
per-edge privacy, and share-consistency recovery, none offered by
existing methods in this setting.

\begin{table}[t]
\centering
\caption{Per-round communication and capability comparison at
$(Q,n_q)=(6,5)$ on UCI HAR
($\checkmark$/$\triangle$/$\times$: native/partial/no).}
\label{tab:secure_agg_compare}
\renewcommand{\arraystretch}{1.05}
\setlength{\tabcolsep}{2pt}
\scriptsize
\resizebox{\columnwidth}{!}{%
\begin{tabular}{@{}L{2.7cm}C{0.95cm}C{0.95cm}C{0.6cm}C{0.7cm}C{0.7cm}@{}}
\toprule
\textbf{Method} & \textbf{bits/coord} & \textbf{DE/round (MB)}
& \textbf{HFL sign-MV} & \textbf{EC privacy} & \textbf{EC recovery} \\
\midrule
Plain signSGD-MV
  & $1$ & $0.14$ & N/A & N/A & N/A \\
Plain FedAvg
  & $32$ & $4.47$ & N/A & N/A & N/A \\
\midrule
HE-FL ($\lambda_{\mathrm{HE}}{=}2{,}048$)
  & $2{,}048$ & $\approx 286$ & $\times$ & $\triangle$ & $\times$ \\
Additive SecAgg~\cite{bonawitz2017practical}
  & $32$ & $4.47$ & $\times$ & $\triangle$ & $\times$ \\
Hi-SAFE~\cite{joo2024hisafe}
  & $\ge\log p$ & multi-round & flat only & $\times$ & $\times$ \\
TEE-HFL~\cite{mo2021ppfl}
  & $32$ (HW) & $4.47$ & $\triangle$ & $\triangle$ & $\triangle$ \\
\midrule
\textbf{\textsc{TierSecAgg}}
  & $\mathbf{30}$ & $\mathbf{4.19}$ & $\checkmark$ & $\checkmark$ & $\checkmark$ \\
\bottomrule
\end{tabular}%
}
\end{table}

\paragraph*{Computation-time benchmarks}
Table~\ref{tab:compute_bench} reports the single-core wall-clock
computation time per global round of a vectorized NumPy implementation
of all cryptographic operations at the representative
$(Q,n_q)=(6,5)$ with $t_{\mathrm{DE}}=t_{\mathrm{EC}}=1$,
$p_{\mathrm{DE}}=31$, and $p=61$.
The device-side cost covers input sharing, local share aggregation,
the masked-opening share, and the Horner evaluation of the
MV-polynomial share.
The edge-side cost covers the two robust openings, and the cloud-side
cost covers per-dealer MDS decoding, aggregation, and sign extraction.
Both decoding paths belong to \textsc{TierSecAgg}: every round first
attempts the fault-free path (interpolation from $t+1$ points plus
verification of all evaluations) and switches to the fault path
(combinatorial unique decoding within $2e+s\le d_{\min}-1$) only upon
a verification failure, so the per-round cost is that of exactly one
path.
All online costs scale linearly in $d$ and remain below $0.3$\,s per
round even for the $d=454{,}922$ MNIST/FMNIST model on the fault-free
path, with the offline preparation row incurred off the online
critical path as discussed above, so the online per-round
computation is dominated by local model training rather than by the
secure aggregation layer.
The heavier fault-path cost, $2.4$\,s at the cloud for the
MNIST/FMNIST model, is incurred only in rounds where Byzantine errors
or dropouts actually
occur, so the steady-state per-round computation of
\textsc{TierSecAgg} equals the fault-free figures and the fault path
is a contingent cost paid exactly when error correction is needed.

\begin{table}[t]
\centering
\caption{Per-round single-core computation time (ms) of
\textsc{TierSecAgg} at $(Q,n_q)=(6,5)$,
$t_{\mathrm{DE}}=t_{\mathrm{EC}}=1$.
}
\label{tab:compute_bench}
\renewcommand{\arraystretch}{1.05}
\setlength{\tabcolsep}{4pt}
\footnotesize
\begin{tabular}{@{}lcc@{}}
\toprule
\textbf{Operation} & \textbf{UCI HAR} & \textbf{MNIST/FMNIST} \\
\midrule
DE offline preparation (per cluster) & $284$ & $3277$ \\
DE device (online total)   & $8.9$   & $119.0$ \\
DE edge (fault-free round)  & $5.0$   & $73.6$ \\
DE edge (faulty round)      & $18.3$  & $247.9$ \\
EC edge                     & $3.7$   & $70.3$ \\
EC cloud (fault-free round) & $20.1$  & $281.6$ \\
EC cloud (faulty round)     & $181.0$ & $2446.0$ \\
\bottomrule
\end{tabular}
\end{table}

\paragraph*{Online comparison with Hi-SAFE}
To quantify the online gap against the only prior signSGD-MV protocol
on equal footing, we implemented both online paths in the flat setting
$(Q,n_q)=(1,30)$ with $p_{\mathrm{DE}}=31$, identical vectorized NumPy
backends, fault-free openings, and preprocessing excluded from both
measurements.
Beaver-triple evaluation as in Hi-SAFE~\cite{joo2024hisafe} uses $N+1$
openings and $N-1$ online secure multiplications per coordinate,
whereas the single-mask construction uses $2$ openings and none,
reducing the online communication from $2n_q(N{+}1)$ to $4n_q$ field
elements per coordinate (a factor of $15.5$ here) and the measured
single-core online time from $734.9$ to $135.8$\,ms per round on
UCI HAR, while replacing $N+1$ sequential online round trips with $2$.

\paragraph*{EC-tier recovery boundary}
At $(Q,n_q)=(10,3)$ with $t_{\mathrm{EC}}=2$, the accepted code has
distance $Q-t_{\mathrm{EC}}=8$ and unique decoding recovers from any
combination of forwarding errors $e$ and erasures $s$ with
$2e+s\le 7$ (e.g., $e=2$ with $s\le 3$, or $e=3$ with $s\le 1$),
while any $(e,s)$ outside this region leads to abort or failure.
This protocol-event-level boundary parallels the empirically verified
DE-tier boundary of Section~\ref{sec:exp_de_recovery}.


\end{document}